\documentclass[sigconf]{acmart}

\usepackage{balance}

\usepackage[ruled,linesnumbered]{algorithm2e}

\newtheorem{mydef}{Definition}

\usepackage{graphicx}
\usepackage{subcaption}
\usepackage{textcomp}
\usepackage{xcolor}
\usepackage{booktabs} 
\usepackage{multirow}

\copyrightyear{2020}
\acmYear{2020}
\setcopyright{acmcopyright}\acmConference[CCS '20]{Proceedings of the 2020 ACM SIGSAC Conference on Computer and Communications Security}{November 9--13, 2020}{Virtual Event, USA}
\acmBooktitle{Proceedings of the 2020 ACM SIGSAC Conference on Computer and Communications Security (CCS '20), November 9--13, 2020, Virtual Event, USA}
\acmPrice{15.00}
\acmDOI{10.1145/3372297.3423338}
\acmISBN{978-1-4503-7089-9/20/11}

\begin{document}

\fancyhead{}

\title{DeepDyve: Dynamic Verification for Deep Neural Networks} 
%
\author{Yu Li$^*$, Min Li$^*$, Bo Luo, Ye Tian, and Qiang Xu}

\affiliation{
  \institution{\underline{CU}hk \underline{RE}liable Computing Laboratory (CURE) \\
  Department of Computer Science and Engineering \\
  The Chinese University of Hong Kong, Shatin, N.T., Hong Kong \\
  Email: \{yuli, mli, boluo, tianye, qxu\}@cse.cuhk.edu.hk}
}
\thanks{$^*$These two authors contributed equally.}
\renewcommand{\authors}{Yu Li, Min Li, Bo Luo, Ye Tian and Qiang Xu}




\begin{abstract}

Deep neural networks (DNNs) have become one of the enabling technologies in many safety-critical applications, e.g., autonomous driving and medical image analysis. DNN systems, however, suffer from various kinds of threats, such as adversarial example attacks and fault injection attacks. While there are many defense methods proposed against maliciously crafted inputs, solutions against faults presented in the DNN system itself (e.g., parameters and calculations) are far less explored. In this paper, we develop a novel lightweight fault-tolerant solution for DNN-based systems, namely \textbf{DeepDyve}, which employs pre-trained neural networks that are far simpler and smaller than the original DNN for dynamic verification. The key to enabling such lightweight checking is that the smaller neural network only needs to produce approximate results for the initial task without sacrificing fault coverage much. 
We develop efficient and effective architecture and task exploration techniques to achieve optimized risk/overhead trade-off in DeepDyve. Experimental results show that DeepDyve can reduce 90\% of the risks at around 10\% overhead.
\end{abstract}



\maketitle
\section{Introduction}
\label{sec:intro}

Machine learning with \textit{deep neural networks} (DNNs) can produce results that have surpassed human-level performance in many challenging tasks lately, and it keeps improving. Consequently, DNNs have become one of the foundation techniques in artificial intelligence (AI) applications. Lots of them (e.g., autonomous driving and medical image analysis) are safety- and security-sensitive. 

Many researchers in the machine learning community believe that DNNs are rather robust to faults~\cite{lecun1990optimal, han2015deep}, wherein removing some neurons or parameters leads to a graceful degradation in model accuracy. However, practical faults do not manifest themselves as the elimination of individual weights or neurons. Instead, they lead to bit-flips on DNN parameters or activations.  Li et al. ~\cite{li2017understanding} conducted a case study on the Eyeriss DNN accelerator~\cite{chen2016eyeriss} under transient faults.
Their results show that the FIT rate caused by random soft errors is far beyond the one required by the ISO 26262 standard (10 FIT\footnote{Failure-in-Time Rate: 1 FIT = 1 failure per 1 billion hours.})~\cite{iso26262} for the functional safety of road vehicles. Comparing to random errors, transient faults caused by malicious attacks are more severe. A recent attack named \emph{DeepHammer}~\cite{yao2020deephammer} shows that it can successfully tamper DNN inference behavior in practical setup, wherein the accuracy of multiple DNN classification systems are reduced to be as low as random guess within a few minutes, by leveraging the rowhammer vulnerability of DRAM used in the system. 



There are a few recent research towards fault-tolerant DNN system designs. In~\cite{li2017understanding}, Li \emph{et al.} proposed to set up a simple threshold to detect those faults that lead to drastic changes in DNN parameters. Such symptom-based error detectors have very little hardware overhead, but their detection capabilities are quite limited, especially for those quantized DNNs used in safety-critical embedded systems. They are also not applicable to defend against malicious faults~\cite{zhao2019fault, rakin2019bit, yao2020deephammer}. 
To tackle this problem, a replication-based error detection technique for DNN systems was proposed in~\cite{li2019d2nn}, but its overhead is quite high, requiring 40\% extra computation to reach 60\% fault coverage on CIFAR-10 dataset. 

This paper aims for a solution with sufficient fault coverage yet little overhead for DNN-based classification systems.
The term \textit{dynamic verification} was introduced in~\cite{austin1999diva}, which detects errors in a complex super-scalar core by checking it with a core that is architecturally identical but micro-architecturally far simpler and smaller. Unlike symptom-based error detection techniques that check for abnormal behaviors, dynamic verification techniques check for in-variants in the system that are nearly always true over all possible executions.
Following this idea, the proposed solution, namely \emph{DeepDyve}, deploys a small neural network (referred to as the \emph{checker DNN}) to approximate the original complex DNN model (referred to as the \emph{task DNN}), and checks whether they produce \emph{consistent} outputs in an end-to-end manner. If their results do not match, re-computation on the task DNN is performed for potential fault recovery, if any. 



Unlike~\cite{austin1999diva}, the result produced by the small checker DNN for error detection is only an approximate result. Hence, there will be both \emph{false positives} (i.e., flagging nonexistent failures) and \emph{false negatives} (i.e., miss to report failures) with DeepDyve. False positives result in unnecessary re-computation cost, while false negatives lead to fault coverage loss. Consequently, it is essential to achieve high \textit{consistency} between the task DNN and the checker DNN under the fault-free situation. We formulate the checker DNN design problem as a design exploration problem, wherein we evaluate a set of candidate small DNN designs for the given big task DNN model and choose the one with optimized coverage/overhead trade-off 
as the checker DNN. 

On the one hand,  we obtain the set of candidate checker DNNs by compressing and transferring the task DNN model's knowledge through knowledge distillation~\cite{hinton2015distilling}.
On the other hand, the given classification task might be too complicated for the much smaller checker DNNs. Under such circumstances, we allow the checker DNN to perform the classification task with reduced complexity. For example, given a task to classify ten objects wherein two kinds of objects are easily confused, we could allow the checker DNN to perform a simpler problem with these two objects treated as one class. By doing so, there will be much less false positives in dynamic verification, at the cost of more false negatives since we would not be able to identify those faults that result in misclsssification between these two kinds of objects. Consequently, we need to carefully perform task simplification for the checker DNN design to strike an optimized balance between coverage and overhead.


From the safety and security perspective, misclassifying different classes often has quite different impacts. For example, in a traffic sign recognition system, misclassifying a "Yield" sign as a "Stop" sign does not cause much trouble, but the opposite misclassification may cause severe traffic accidents. We need to consider such risk impacts in the checker DNN design. That is, we should be more concerned about the risk/overhead trade-off instead of coverage/overhead trade-off.

The proposed DeepDyve solution considers the above issues, and the main contributions of this paper include:
\begin{itemize}
\item To the best of our knowledge, DeepDyve is the first dynamic verification technique for resilient DNN designs, which uses a far simpler and smaller checker DNN for online error detection and recovery. 

\item We propose a novel two-stage checker DNN design methodology, which explores both the checker DNN architectures and task simplification possibilities. In particular, we propose a novel checker DNN architecture exploration technique with theoretical guarantees. Also, being able to manipulate the tasks performed on the checker DNN dramatically increases the solution space of DeepDyve.


\item DeepDyve leverages the uneven risk probabilities and safety impact among classes to guide the design exploration procedure. Experimental results on CIFAR-10, GTSRB, CIFAR-100, and Tiny-ImageNet datasets show that it can reduce up to 90\% of the risks at around 10\% computational overhead.  


\end{itemize}


We organize the remainder of this paper as follows. Section~\ref{sec:backgorund} presents the preliminaries and motivation of this work. Then, we give an overview of DeepDyve in Section~\ref{sec:overview}. Next, we detail the checker DNN architecture exploration and task exploration techniques in Section~\ref{sec:archi} and~\ref{sec:task}, respectively. Experimental results are presented in Section~\ref{sec:experiment}. Then, we discuss the applications and limitations of DeepDyve in Section~\ref{sec:discussion}, followed by the survey of related works in Section~\ref{sec:related}. Finally, Section~\ref{sec:conclusion} concludes this paper.

\section{Preliminaries and Motivation}


In this section, we first present the impact of hardware faults (random
or malicious) on the reliability and security of DNN systems in Section \ref{subsec:nnunderfaults}. Then, we provide the threat model in Section \ref{subsec:threatmodel}. At last, we illustrate 
the motivation for the proposed DeepDyve solution in Section \ref{subsec:motivation}.

\label{sec:backgorund}

\subsection{DNN Systems under Faults}
\label{subsec:nnunderfaults}

There are mainly two types of attacks during the inference of DNN systems: adversarial example attack and fault injection attack. Adversarial example attacks~\cite{GoodfellowSS14, MadryMSTV18} try to fool DNN systems by crafting subtle malicious perturbations on inputs, and there is a vast body of research on this topic~\cite{KurakinGB17, luo2018towards, papernot2017practical, meng2017magnet}. By contrast, fault injection attack aims to break the system by injecting faults into the internal system execution pipeline, e.g., flipping data bits in processing elements. This problem is less explored in the literature, and the overhead of existing defense techniques is quite high.   

Various types of transient faults can be introduced into DNN systems by attackers, e.g., 
clock glitch attack \cite{matsubayashi2016clock}, voltage glitch attack \cite{romailler2017practical}, and rowhammer attack~\cite{kim2014flipping}. 
Besides malicious faults, transients faults could also occur due to environmental perturbations such as alpha particle strikes. 


\vspace{5pt}
\noindent
\textbf{The Impact of Transient Faults.}
Transient faults may occur at the data paths and buffers of processing units~\cite{li2017understanding} or inside the memories of DNN-based systems~\cite{hong2019terminal}. Such faults would manifest themselves as errors in DNN calculations or intermediate values~\cite{reagen2018ares} during inference. A failure occurs when errors propagate to the outputs of the system and cause behavioral changes, which could lead to catastrophic consequences in safety-critical applications, e.g., misclassifying a "Stop" sign as a "Yield" sign in autonomous vehicles and taking the wrong action. In the following, if not specified, faults, errors and failures are used in an exchangeable manner, and we are only concerned with those faults that cause misclassifications. 

\vspace{5pt}
\textbf{Existing Attacks.}
Recently, DNNs are shown to be vulnerable to fault injection attacks~\cite{liu2017fault}. In
\cite{hong2019terminal}, the authors estimated that 40--50\% of the parameters in a DNN model could lead to an accuracy drop greater than 10\% when bit-flips occur in their data representation. While it was shown that quantized DNNs are more resilient to fault injection attacks, the recent progressive BitFlip Attack (BFA) \cite{rakin2019bit} proposed by Rakin et. al. can reduce the accuracy of a quantized ResNet-18 from 68.9\% to 0.1\% with only 13 bit-flips. BFA combines gradient ranking and progressive search to identify those vulnerable bits that degrade model accuracy significantly when flipped. The follow-up work DeepHammer by Yao \textit{el al.} ~\cite{yao2020deephammer} proves the effectiveness of BFAs in practice, by row-hammering against various real DNN systems.


\begin{figure}[t]
    \centering
    \includegraphics[width=0.8\linewidth]{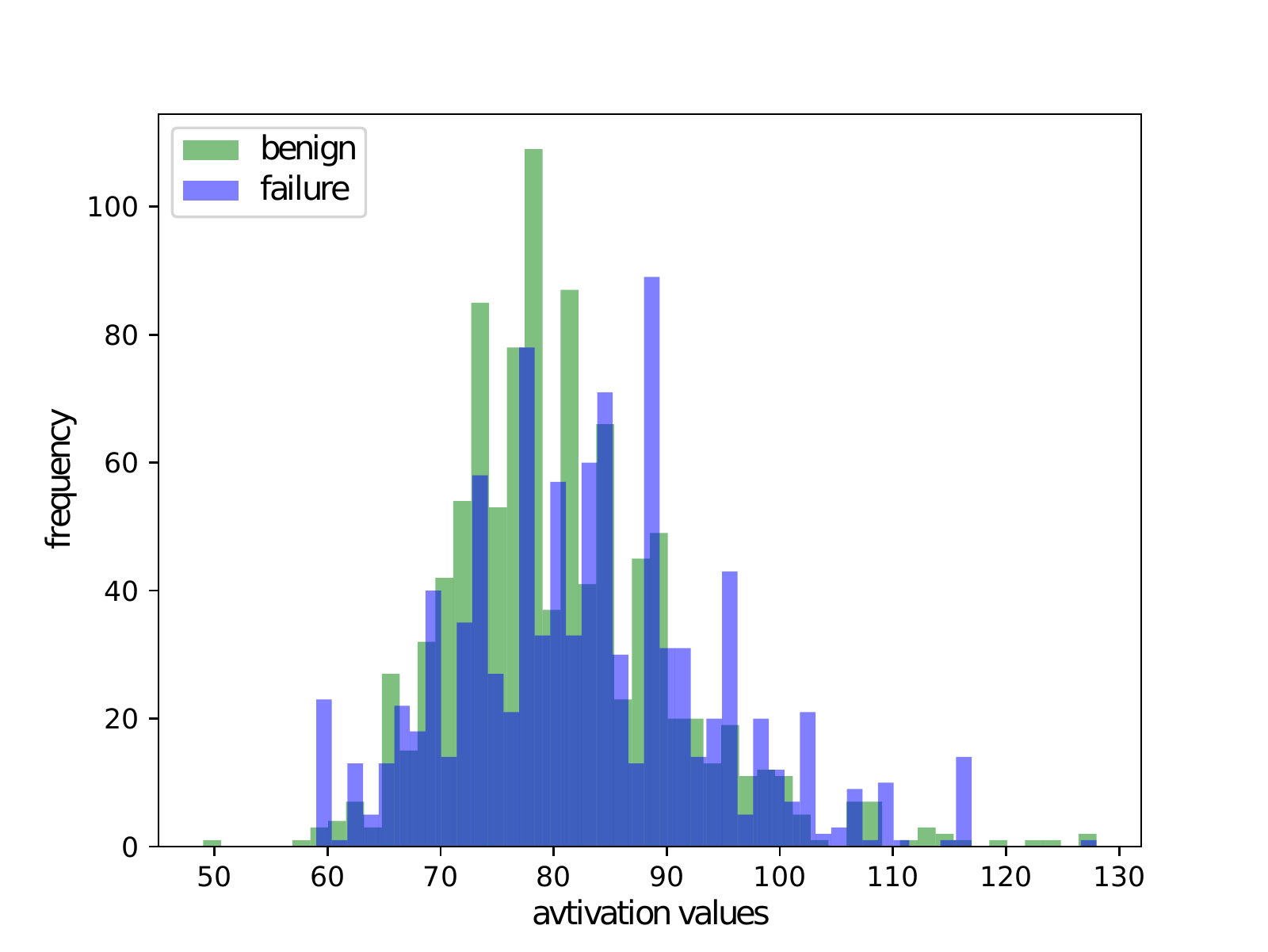}
    \caption{Comparison of the maximum activation value  between benign and failure cases (because of fault attack) from a quantized (8-bit integer) classifier. 
    }
    \label{fig:maxActivation}
\end{figure}

\vspace{5pt}
\textbf{Existing Defenses.}
For DNNs implemented on floating-point machines, only a small fraction of the dynamic range provided by the data type is used. 
If a fault makes the magnitude of intermediate output values huge, it is likely to lead to a failure. Based on this observation, Li \emph{et al.} proposed to use a simple threshold to detect those faults that lead to intermediate outputs beyond it~\cite{li2017understanding}. To be specific, before deployment, they record the value ranges $(X_{min}, X_{max})$ of the output for each layer. After deployment, the output value range is then checked in the run-time. They consider a fault is detected if there are output values beyond the range $(1.1 \times X_{min}, 1.1 \times X_{max})$. 
This anomaly detector has very little hardware overhead, but its detection capabilities are quite limited, especially for recent attacks on quantized DNN systems.
We show an example in Figure ~\ref{fig:maxActivation} where the threshold-based anomaly detector fails to detect faults in quantized classifiers. For the threshold detector to detect faults, at least the maximum activation value should be beyond the normal range. However, our experiment shows that the maximum activation values from all failure cases (due to hardware fault) are not bigger than the normal boundary. 

Li \emph{et al.}~\cite{li2019d2nn} design a replication-based error detection technique for deep neural networks. However, their overhead is quite high. They spend 40\% overhead to reach 60\% fault coverage on CIFAR-10.
In the DeepHammer work~\cite{yao2020deephammer}, the authors discuss a few potential mitigation techniques but do not provide any quantitative results for their effectiveness.

\subsection{Threat Model}
\label{subsec:threatmodel}
In this work, we consider the attacker is trying to compromise the accuracy of a DNN system by maliciously injecting faults into it. Unlike crafting malicious inputs to fool DNN systems, we target faults presented in the system internals, i.e., processing elements, buffered weights, and intermediate values stored in on-chip buffers or memories, and so on. We consider the attacker succeeds if the model's output class is different from the one obtained in an attack-free environment.

We assume the attackers have full knowledge of the DNN and its deployment on the device, including neural network topology, parameters, and low-level implementation details, e.g., the position of intermediate values stored in memories. Weak attackers could launch random bit-flips, while for strong attackers, they can precisely locate and launch fault injection in the processing pipeline.

Moreover, we assume the same transient faults would not occur in consecutive DNN inference runs. Firstly, reliability threats rarely occur and the probability to occur repeatedly in a short period is negligible. Secondly, it is very difficult, if not impossible, to launch the same faults repeatedly in a DNN system.  
For example, in \cite{yao2020deephammer}, launching a rowhammer attack  requires long preparation time (several minutes). As long as the DNN inference time is short (and usually it is), attackers do not have sufficient time to launch the same attack in the second run.

\subsection{Motivation}
\label{subsec:motivation}
For any classification problem, there could be many DNN models with different size/accuracy trade-offs to solve it.
Although big models often have higher accuracy, many inputs can be correctly handled by small models.
Therefore, we could employ a smaller checker DNN to perform the same task, and they should output the same results in \emph{most} cases when faults do not occur. In this way, the task model can be dynamically verified for online error detection and recovery. Note that, it is not possible to achieve deterministic dynamic verification for such systems because the outputs of a simple model cannot achieve 100\% consistency with that of the original model.

As discussed earlier, because the checker DNN is less accurate, there will be false positives and false negatives. 
Generally speaking, the larger the checker DNN is, the more accurate it is~\cite{tan2019efficientnet}, but it does not necessarily lead to larger computational overhead. We use the following example to illustrate the impact of checker DNN design.

Suppose the task DNN model is with $\sim$100\% accuracy, 
and there are three candidate checker DNN models: A, B, and C. Their relative sizes compared to the task DNN and their classification accuracy are shown in the second and the third column of Table~\ref{tab:motivation}. Consider transient fault-induced failures are rare events, the computational overhead of the three models is estimated in the fourth column, which is the sum of the computational cost of the checker DNN itself (as it is always on) and the re-computation cost on the task DNN when the checker DNN produces a different classification result (false positive cases). 

\begin{table}
\caption{A Motivational example.}
\label{tab:motivation}
\centering
\begin{tabular}{|c|c|c|c|}
\hline
Checker DNN & Size  & Accuracy & Computational Overhead   \\ \hline
A           & 1\%   & 80\%     & $\sim$21\%  \\ \hline
B           & 5\%   & 92\%     & $\sim$13\%    \\ \hline
C           & 10\%  & 94\%     & $\sim$16\%   \\ \hline

\end{tabular}
\end{table}

\begin{figure*}[t]
    \centering
    \includegraphics[width=0.95\linewidth]{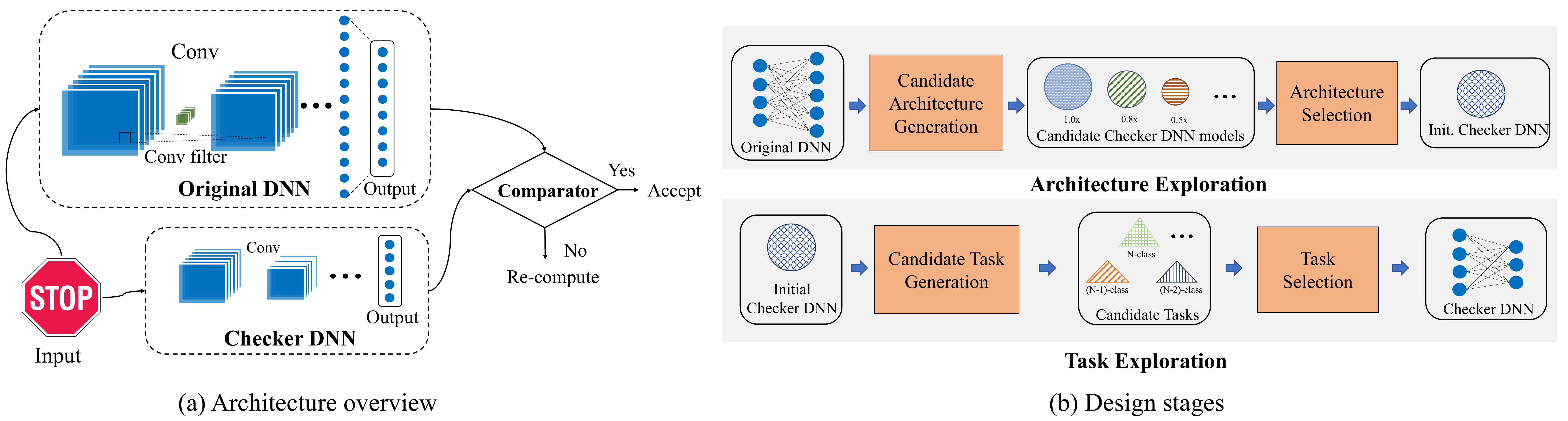}
    \caption{DeepDyve architecture and the design flow.}
    \label{fig:overview}
\end{figure*}

If the misclassifications caused by faults are evenly distributed among all classes, the fault coverage would be similar to the accuracy of the checker DNN. For this particular example, checker B can achieve $\sim$92\% fault coverage with 5\% hardware cost and $\sim$13\% computational cost, which significantly outperforms existing fault-tolerant solutions for DNN designs. This has motivated the proposed DeepDyve solution in this paper. 

In practice, fault-induced misclassifications are usually not evenly distributed. Moreover, misclassifying different classes often have different risk implications for safety-critical systems. Therefore, for any possible checker DNN, we need to consider the risk impact and apply fault simulation to evaluate its effectiveness. This is a time-consuming process. Therefore, the critical challenge is how to efficiently explore the solution space of all possible checker DNN designs to find 
the one with an optimized risk/overhead trade-off. 

\section{DeepDyve Overview}
\label{sec:overview}

Figure~\ref{fig:overview} (a) depicts the proposed DeepDyve architecture. It contains three parts: the original network, the checker DNN, and the comparator. The checker DNN is a smaller and simpler DNN model which approximates the original network and the original task. An input instance is processed by both of the two models. Their outputs are checked by the comparator and accepted if consistent. Otherwise, the input instance is subject to a re-computation by the original model, and the new prediction is accepted, regardless of whether the two model outputs are consistent or not.

\noindent
\textbf{Design Goals.}
We first formally define the evaluation metrics (e.g. coverage and overhead).
\label{subsec:metrics}
We evaluate the cost of checker DNN design by the introduced overhead. The overhead of the checker DNNs can be calculated as follows:
\begin{equation}\small
    O(S) = \frac{params_{small}}{params_{big}},
    \label{equ:overheads}
\end{equation}
\begin{equation}\small
    O(C) = \frac{FLOP(net_{small}) + (1-P_{consistent}) \times FLOP(net_{big})}{FLOP(net_{big})},
\label{equ:overheadc}
\end{equation}

$O(S)$ and $O(C)$ stand for the storage overhead and computational overhead, respectively, wherein $params$ stands for the storage requirement of model parameters with unit of Mega Bytes ($MB$), and $FLOP$ function calculates the number of multiply-accumulation operations in the network. The computational overhead contains two parts: FLOP of the small network (static overhead) and the re-computation overhead (dynamic overhead) when the small network output is different from that of the big one (with probability $ 1- P_{consistent}$). 

The detection ability of DeepDyve is characterized by the coverage rate. A popular definition of fault coverage would be number of classification failures detected among all mis-classifications caused by faults, as show in in Equation~\ref{equ:coverage}. $DF_{i,j}$ stands for the detected failures mis-classified from class $i$ to class $j$, and $TF_{i,j}$ denotes the total failures from class $i$ to class $j$ when faults occur. To take the different risk impact of different failures on safety-critical application into consideration, we introduce a new metric called \emph{weighted coverage}, abbreviated as $WCov.$ in Equation~\ref{equ:weightedCoverage}. $I_{i,j}$ is the risk impact if class $i$ is mis-classified into class $j$, which will be defined later in this Section. Note that $Cov.$ is a special case of $WCov.$ when all misclassifications have the same risk impact. In later text, we use coverage and weighted coverage interchangeably and they both refer to weighted coverage if not specified. 
\begin{equation}\small
    Cov. = \frac{\sum_{ij} DF_{ij}}{\sum_{ij}TF_{ij}}.
    \label{equ:coverage}
\end{equation}

\begin{equation}\small
    WCov. = \frac{\sum_{i j} DF_{ij} \times I_{i j}}{\sum_{ij} TF_{ij} \times I_{ij} }, \forall i, j \in N \ \text{and}\  i \neq j.
    \label{equ:weightedCoverage}
\end{equation}

\noindent
\textbf{Design Stages.}
Under the guidance of the design goals, there are mainly two stages in designing of the checker DNN and we show them in Figure~\ref{fig:overview} (b). The first stage is architecture exploration, where we initialize the architecture of the checker DNN. Given a task model, a pool of checker DNN candidates with the same task of the given model are generated with model compression techniques. Then, one of them is picked from the pool by evaluating their overhead and fault coverage, detailed in Section~\ref{sec:archi}. The second stage is task exploration, where we try to manipulate the classification tasks performed by the checker DNN to achieve better coverage/overhead trade-off. That is, we can find a better solution by providing more design options at the task level, detailed in Section~\ref{sec:task}.

To solve the above design exploration problems, we define the following three matrices:

\begin{figure}
    \centering
    \includegraphics[width=0.75\linewidth]{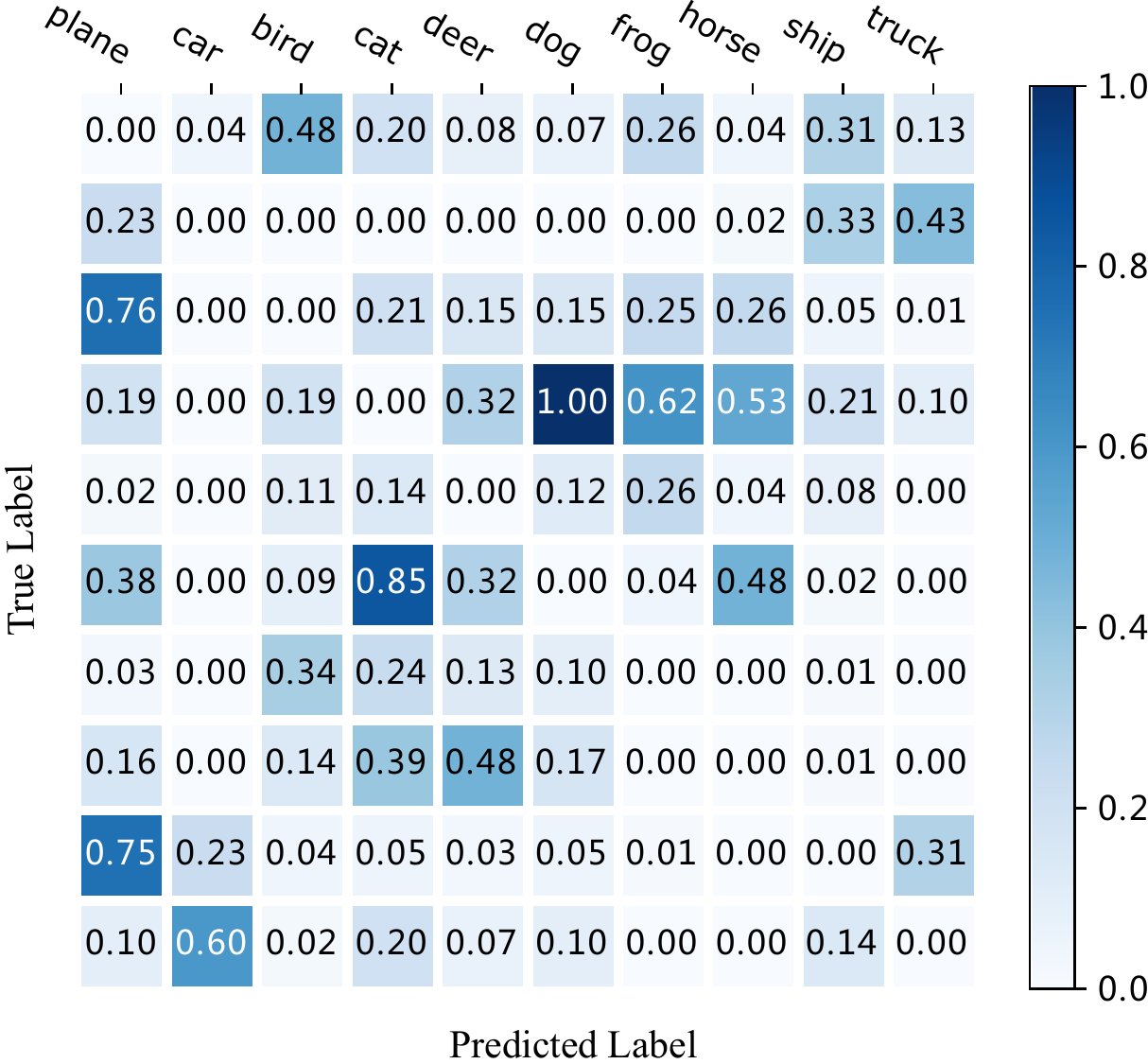}
    \caption{An example of risk probability matrix obtained from CIFAR-10 (normalized for visualization).}
    \label{fig:vulMatrix}
    \vspace{-10pt}
\end{figure}

\begin{itemize}
    \item 
    \textit{Risk impact matrix} $\mathbf{I} \in \mathbb{R}^{\mathbf{N}\times\mathbf{N}}$. In safety-critical DNN applications, the risk impact of different misclassifications may vary significantly from the system perspective. Each entry in  $\mathbf{I}$ denotes the cost of the corresponding misclassification (the larger the value, the higher the cost). 
    As the actual risk impact depends on the application, the values in the impact matrix should be carefully determined by system designers. 
    \item \textit{Risk probability matrix} $\mathbf{R} \in \mathbb{R}^{\mathbf{N}\times\mathbf{N}}$, where the entry $R_{ij}$ denotes the probability that the $i$-th class is misclassified to $j$-th class when faults occur, and $N$ denotes the total number of classes. Risk probability matrix is obtained from fault injection experiments. Figure~\ref{fig:vulMatrix} shows an example of $\mathbf{R}$ drawn from CIFAR-10 dataset by performing random fault injection on VGG-16 for 400,000 times. 
    From this example, we can observe fault-induced misclassifications are far from evenly distributed.  
    \item
    \textit{Inconsistency matrix} $\mathbf{C} \in \mathbb{R}^{\mathbf{N}\times\mathbf{N}}$, wherein each entry $C_{ij}$ denotes the probability that one sample is labeled as $i$-th class by the task model while is labeled as $j$-th class by the small model in DeepDyve. Note that entry $C_{ii}, i\in {[1..N]}$ equals to zero.
    In practice, some classes are naturally more difficult to classify (e.g., \textit{dog} and \textit{cat} in the CIFAR-10 dataset) than others. These difficult classes cause more inconsistency than the easy ones. Combining them as one class (e.g., as \textit{pet}) in the checker DNN is relatively easy to achieve high consistency. 
\end{itemize}

Risk impact matrix $\mathbf{I}$ will be used in calculating weighted coverage. Besides, the three matrices $\mathbf{I}$, $\mathbf{R}$ and $\mathbf{C}$ will be all used in task exploration, wherein we try to combine those classes that are easily confused yet have less risk for task simplification. 



\section{Architecture Exploration}
\label{sec:archi}
The objective of the architecture exploration procedure is to find a initial checker DNN model that achieves good fault coverage with low overhead defined in Section~\ref{subsec:metrics}. 
To this end, firstly, we generate a pool of checker DNN candidates. The generation process is trying to minimize the overhead with the help of model compression techniques proposed in~\cite{howard2017mobilenets}, detailed in Section~\ref{subsec:candidate}.
Second, as different candidates offer different trade-offs between overheads and coverage, we illustrate how to efficiently search for an appropriate checker DNN design from the candidates in Section~\ref{subsec:searchArchi}. 

\subsection{Checker DNN Candidate Generation}
In DeepDyve, the \textit{consistency} between predictions of the checker DNN and those of the original DNN decides the computational overhead $O(C)$. Consider an input that is mis-classified by the original DNN when no faults occur, we would like to have the checker DNN output the same wrong label, so that DeepDyve does not flag a nonexistent failure, avoiding unnecessary re-computation. 

To improve consistency, given the task DNN model, we use model compression to generate the checker model candidates. Specifically, we use two types of model compression techniques. First, we use architecture compression to search for the potential architectures and then we use knowledge distillation to train our checker DNN. 
\label{subsec:candidate}

\vspace{5pt}
\textbf{Architecture Compression.}
No doubt to say, the amount of available design choices has a significant impact on any design exploration problem. In order to increase the design options for DeepDyve, we adopt the model compression approach in~\cite{howard2017mobilenets} to make the size of checker DNNs adjustable. To be specific, given the task DNN architecture, we use a single width multiplier $\alpha$ to adjust it, by uniformly scaling down the number of channels (or neurons if it is a linear layer) for each layer. For example, a feature map with 100 channels will be scaled down to the one with ten channels with $\alpha$ being set to 0.1. By applying width multiplier, the resulting model architecture has much less overhead. 


\begin{table}[t]
\caption{ResNet-10 with different width multiplier.}
\label{tab:width}
\centering
\begin{tabular}{|c|c|c|c|}
\hline
$\alpha$ & Accuracy(\%) & O(S) (MB) & O(C) (GFLOPs) \\ \hline
1.0   & 97.54        & 1.23      & 0.06         \\ \hline
0.7   & 96.94        & 0.60      & 0.03         \\ \hline
0.5   & 96.20        & 0.31      & 0.02         \\ \hline
0.3   & 95.75        & 0.12      & 0.01         \\ \hline
\end{tabular}
\end{table}

We take one of popular architectures---ResNet trained on GTSTB~\cite{Stallkamp2012} as a case study to show the effect of width multiplier. Table~\ref{tab:width} lists the accuracy, the storage overhead (in MegaByte) and the computational overhead (in Giga Floating Point Operations) of ResNet-10 with different width multipliers. The first row stands for the original ResNet-10. As can be observed, accuracy drops smoothly with smaller model size and less computational cost. 

\vspace{5pt}
\textbf{Parameter Training.}
\label{subsec:training}
To further improve the consistency between the task DNN and the checker DNN, we use \textit{knowledge distillation} to train the checker DNN. Knowledge distillation, formulated by Hinton \textit{et al.}~\cite{hinton2015distilling}, is a training solution to distill a task model (teacher model) and transfer knowledge to a simpler model (student model). 

In our training, the first step of knowledge distillation from the task DNN is to covert the pre-softmax logits, $z_i$, computed for each class into a probability, $p_i$, by Equation~\ref{equ:kd} with the \textit{temperature} $T$. 
\begin{equation}\small
    p_i=\frac{exp(z_i/T)}{\sum_j exp(z_j/T)}
    \label{equ:kd}
\end{equation}
With higher temperature, the new targets for the checker DNN to learn are `softer' probability distributions over classes. 

Next, the checker DNN is trained by minimizing the knowledge distillation loss ($L_{KD}$), which is defined as:
\begin{multline}\small
    L_{KD}=\lambda T^2 \times CrossEntropy(P_{C}^{T}, P_{O}^{T}) 
    \\ + (1-\lambda)\times CrossEntropy(P_C, y_{true}),
    \label{equ:KDLoss}
\end{multline}
wherein $P_{C}^{T}$ and $P_{O}^{T}$ are the softened outputs of the checker DNN and the original DNN under the same temperature $T$. The first component of $L_{KD}$ forces the checker DNN towards approximating similar output distribution of the original DNN (i.e., consistency), whereas the second component of $L_{KD}$ forces the checker DNN towards correctly classifying inputs as usual (i.e., accuracy). We use $\lambda$ to tune the weighted average between the kinds of losses.

\begin{figure*}[t]
    \centering
    \includegraphics[width=1\linewidth]{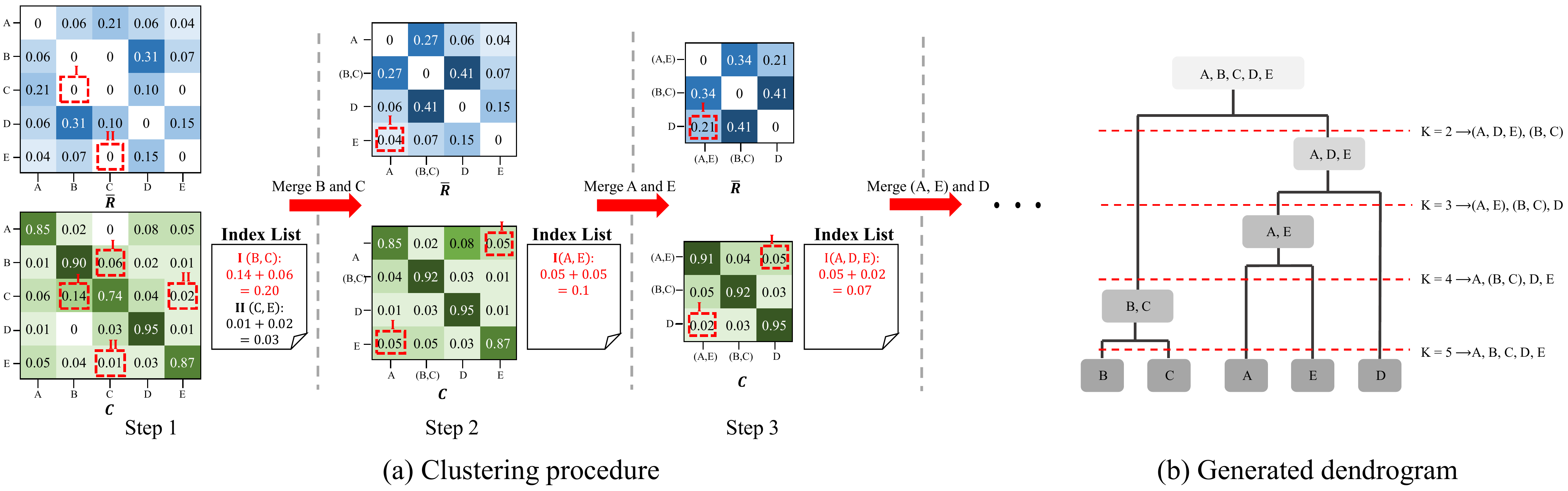}
    \caption{An illustrative example of agglomerative class clustering.}
    \label{fig:algExample}
\end{figure*}

\subsection{Search Strategy}
\label{subsec:searchArchi}

Our search strategy is based on the empirical observation that the consistency between the two models in DeepDyve is related to the multiplier $\alpha$ used to generate the small checker DNN ($\alpha<<1$). 
Here we formally define the \textit{consistency function} of $\alpha$ by:

\begin{mydef}
     Let $g_{t}(\mathbf{x})$ and $g_{\alpha}(\mathbf{x})$ be the task DNN model and the checker DNN model generated from the task DNN with a multiplier $\alpha$. Each model $g$ takes a vector of input $\mathbf{x}$ and outputs a class of $y \in {(0..N-1)}$, where $N$ is the number of classes. We define the consistency function as the probability that the outputs are consistent over the input space $\mathcal{X}$.
    \begin{equation}
    f(\alpha) = P (g_{t}(\mathbf{x}) = g_{\alpha}(\mathbf{x})), \mathbf{x} \in \mathcal{X}
    \end{equation}
\end{mydef}

We choose a simple function with the form of $f(\alpha) = -\frac{a}{\alpha} + b$ to approximate the consistency function. 
In this function, $a$ and $b$ are positive parameters whose values can be obtained via curve fitting after collecting a number of checker DNN models with different $\alpha$ and the corresponding consistency values. Besides, because $0 \leq f(\alpha) \leq 1$ and $ 0 < \alpha \leq 1$, we can obtain the valid range of $\alpha$, which is $\frac{a}{b} \leq \alpha \leq 1$.


Once $a$ and $b$ are known, we have the following theorem: 

    

\begin{theorem}\label{theorem:optimal}
Suppose $f(\alpha)$ is approximated with $-\frac{a}{\alpha} + b$, when $a>0$, $b>0$, and $\frac{a}{b} \leq \alpha \leq 1$. We can find an optimal $\alpha = \sqrt[3]{\frac{a}{2}} $ where the computational overhead $O(C)$ is minimized.
\end{theorem}

\begin{proof}
First, for a neural network composed of linear and convolutional layers, the FLOPs of the checker DNN with a multiplier $\alpha$ is $\alpha ^2$ times the original FLOPs.
Recall that the number of floating point operations (FLOPs) for one linear layer can be estimated by:
\begin{equation}
   2 \times I \times O
\end{equation}
where $I$ and $O$ are number of input and output neurons in one linear layer, respectively. Therefore, the FLOP of a compressed linear layer with multiplier $\alpha$ is:
\begin{equation}
    \alpha ^2 \times (2 \times I \times O)
\end{equation}

Similarly, for a convolutional layer, the floating point operations with multiplier $\alpha$ is estimated by:
\begin{equation}
    (2 \times k^2 \times C_{in}) \times (H_{out} \times W_{out} \times C_{out})
\end{equation}
where $k$ stands for the kernel size, and $C_{in}$, $C_{out}$ stands for number of input and output channels, respectively. $H_{out}$ and  $W_{out}$ are the height and the width of the output tensors. Given this, the FLOP of a compressed convolutional layer with multiplier $\alpha$ is:  
\begin{equation}
    \alpha ^2 \times (2 \times k^2 \times  C_{in}) \times (H_{out} \times W_{out} \times C_{out}).
\end{equation}
Hence, if we add all layers together, the final FLOPs of the checker DNN will be $\alpha ^2$ times of the task model where $\alpha = 1$.

Providing this, the computational overhead of DeepDyve with the checker DNN can be simplified from Equation~\ref{equ:overheadc} to Equation~\ref{equ:simpleoc}.
\begin{equation}
    O(C) = \alpha ^2 + (1- f(\alpha)), \alpha \in (0, 1]
    \label{equ:simpleoc}
\end{equation}

\begin{equation}
    O(C) = \alpha ^2 + (1+\frac{a}{\alpha} - b), \alpha \in (0, 1], a > 0, b > 0, 
\end{equation}

\begin{equation}
    \nabla O(C) = 2 \alpha  -  \frac{a}{\alpha ^2}, \alpha \in (0, 1], a > 0
    \label{equ:ocgradient}
\end{equation}


\begin{equation}
    \text{Let} \ \nabla O(C) = 0, \text{then} \ \alpha = \sqrt[3]{\frac{a}{2}}
    \label{equ:optimal}
\end{equation}

To find the optimal point, we calculate the gradient of $O(C)$ as Equation~\ref{equ:ocgradient}. By letting the gradient equals to 0, we obtain the optimal point of $\alpha$, which is $\sqrt[3]{\frac{a}{2}}$, as shown in Equation~\ref{equ:optimal}.
\end{proof}

Therefore, to obtain the optimal $\alpha$, we are going to fit the consistency function $-\frac{a}{\alpha} + b$ with the given candidate pool. After that, we select the checker DNN with $\alpha = \sqrt[3]{\frac{a}{2}}$. 



\section{Task Exploration}
\label{sec:task}

After the initial checker architecture is fixed, DeepDyve performs task exploration to achieve better risk/overhead trade-off. In Section~\ref{subsec:simplification}, we first discuss how to perform task simplification efficiently under the guidance of risk probability matrix $\mathbf{R}$, risk impact matrix $\mathbf{I}$ and inconsistency matrix $\mathbf{C}$. The first step generates a bunch of different tasks. Then, we detail the search strategy to select the best task in Section~\ref{subsec:searchTask}.


\begin{algorithm}
\caption{Agglomerative class clustering}
\label{algo:cluster}
\KwIn{Risk matrix $\overline{\mathbf{R}}$, 
inconsistency matrix $\mathbf{C}$, No. of classes $N$,
class labels $L$}
\KwOut{$(N-2)$ cluster label lists}
\tcc{Initialize}
$k = N-1 $\;
$candidateList = []$\;
\For{$q=1, 2, \ldots, N$}
{
    $G_q = l_q$\tcp*{\footnotesize Cluster with single class}
    $\lambda_q = q$\tcp*{\footnotesize Initialize cluster label}
}
\While{$k \geq 2$}{
\tcc{\footnotesize Select clusters based on two criteria}
$indexList =$ all $\mathop{\arg \min}_{i,j}COV_{loss}(i,j)$ in $\overline{\mathbf{R}}$\;
$(n, m) = \mathop{\arg \max}_{i,j} O_{save}(i,j)$ in $indexList$\;
\tcc{\footnotesize Update cluster label list, matrices}
Merge $G_n$ and $G_m$, Update clusters $\{G\}$ and $\boldsymbol{\lambda}$\;
Update $\overline{\mathbf{R}}$, $\mathbf{C}$\;
\tcc{\footnotesize Add cluster label list to candidates}
$candidateList[k] = \boldsymbol{\lambda}$\;
$k = k -1$\;
}

\Return $candidateList$;
\end{algorithm}

\subsection{Agglomerative Class Clustering}
\label{subsec:simplification}
Given the original $N$-class task, our problem is to find a simplified $K$-class task for any given checker DNN with better overhead/coverage trade-off. We consider it as a clustering problem and propose the \textit{Agglomerative Class Clustering} to solve it (see Algorithm~\ref{algo:cluster}). 

Formally, we assume the labels of original $N$ classes as:
$L=\{l_1, l_2, \ldots, l_N\}$. For the sake of simplicity, we can map the labels into integer numbers, as $L=\{1, 2, \ldots, N\}$. They are to be grouped into $K$ clusters $\{G_k \vert k = 1, 2, \ldots, K\}$, where
$G_{k^{\prime}}\bigcap_{k^{\prime}\neq k} G_{k}=\varnothing$ and $L = \bigcup^{K}_{k=1}G_k$. Accordingly, we can use $\lambda_i\in\{1, 2, \cdots, K\}$ to represent the cluster label of original label $l_i$. Then, the clustering result can be represented by a cluster label list: $\boldsymbol{\lambda}=(\lambda_1, \lambda_2, \cdots, \lambda_N)$.


\noindent
\textbf{Risk matrix $\overline{\mathbf{R}}$.}
The risk probability matrix and risk impact matrix can be integrated into one single risk matrix with an element-wise multiplication:
\begin{equation}
\overline{\mathbf{R}} = \mathbf{R}\odot \mathbf{I},
\end{equation}
wherein each entry in $\overline{\mathbf{R}}$ stands for the risk between classes of the big network. 

\noindent
\textbf{Two clustering criteria.} Merging two classes into one cluster have two effects:
\begin{itemize}
    \item Coverage loss, since fault-induced misclassfications between the two classes cannot be detected any more. Hence, we prefer merging classes with small values in risk matrix $\overline{\mathbf{R}}$.
    \item Overhead savings, because the simplified task is easy to learn and it will be more consistent with the big DNN, thereby reducing re-computational overhead. 
\end{itemize}

We use $COV_{loss}$ and $O_{save}$ to denote such effects, which are used in Algorithm 1. 


\noindent
\textbf{Updating $\overline{\mathbf{R}}$ and $\mathbf{C}$.} After selecting two classes or clusters to merge, we should update the $\overline{\mathbf{R}}$ and $\mathbf{C}$ accordingly. Assuming cluster $G_i$ and $G_j$ ($i < j$) are to be merged, we first move classes in cluster $G_j$ to $G_i$, as $G_i = G_i \bigcup G_j$, delete $G_j$ and re-assign the cluster label for the rest of clusters. Then we update the risk  and inconsistency values of $G_i$ ($i$-th row and $i$-th column of $\overline{\mathbf{R}}$ and $\mathbf{C}$) as the sum of the corresponding values of two clusters. Lastly, we delete the $j$-th row and $j$-th column in $\overline{\mathbf{R}}$ and $\mathbf{C}$. In this way, we aggregate the risk probability and inconsistency values of two merged clusters while preserving the property of the matrices defined in Section~\ref{sec:overview}.

\noindent
\textbf{Clustering scheme.} We apply a hierarchical clustering algorithm---\textit{agglomerative class clustering}, which is illustrated in Algorithm~\ref{algo:cluster}. Specifically, each class of the original task starts in its own cluster. Then, we search for two clusters with the smallest coverage loss in $\overline{\mathbf{R}}$. Note that $\overline{\mathbf{R}}$ (see Figure~\ref{fig:vulMatrix}) is usually sparse, in case of multiple occurrences of the minimum value, we choose the one with the largest overhead savings from $\mathbf{C}$, which in turn improves the model consistency between the task model and the checker DNN. Next, we merge the two selected clusters into one cluster and update $\overline{\mathbf{R}}$ and $\mathbf{C}$. The above procedure iterates in a bottom-up manner until all classes of the original task are merged as a single cluster. Consequently, the clustering results can be presented in a dendrogram with $K=2$ to $K=N-1$ clustering candidates, which enables later exploration for the optimal simplified task for a given checker DNN. Figure~\ref{fig:algExample} shows an example of the iterative clustering procedure with five classes and the generated dendrogram.


\subsection{Search Strategy}
\label{subsec:searchTask}


After obtaining candidate tasks from Algorithm 1, we
evaluate the corresponding overhead savings and coverage loss with fault injection experiments. As the number of candidate tasks $N-1$ increases linearly with the number of original classes, we can evaluate them efficiently.  Afterwards, we have a list of pareto-optimal checker DNN designs with various coverage/overhead trade-offs. We then choose the optimal final DNN design.



\begin{figure*}[t]
    \centering
    \includegraphics[width=0.8\linewidth]{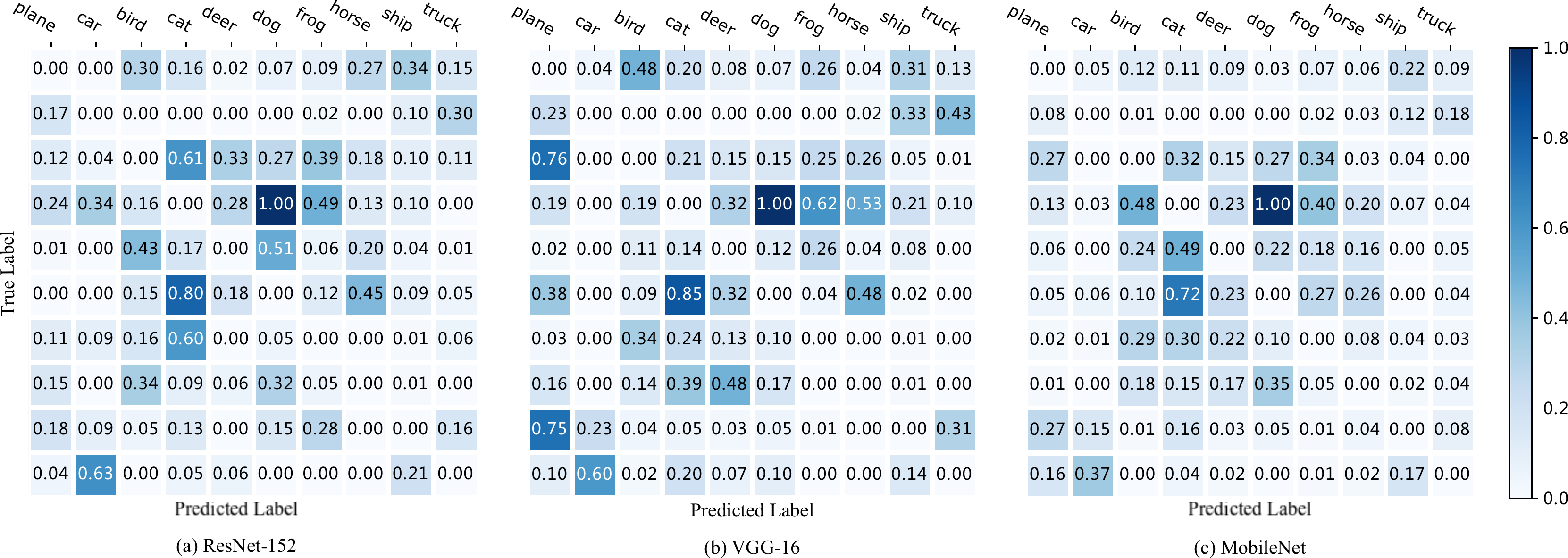}
    \caption{Risk probability matrices of different models under four million fault injections.}
    \label{fig:vultimes}
\end{figure*}

\section{Experimental Results}
\label{sec:experiment}

In this section, we demonstrate the effectiveness of DeepDyve. First, we present our experimental setup in Section \ref{subsec:setup}.
After that, our results that architecture exploration facilitates to find an optimized checker DNN architecture in Section \ref{subsec:effectArchi}, and task simplification further saves the overhead of DeepDyve in Section \ref{subsec:effectTask}, respectively. We show DeepDyve outperforms existing solutions in Section \ref{subsec:comparisonthresh}. At last, in Section \ref{subsec:casestudy}, we discuss the impact of model accuracy through a case study on CIFAR-100.

\subsection{Setup}
\label{subsec:setup}
\textbf{Datasets.}
We demonstrate the effectiveness of DeepDyve on four widely used image classification datasets: CIFAR-10~\cite{krizhevsky2009learning}, The German Traffic Sign Recognition Benchmark (GTSRB)~\cite{Stallkamp2012}, CIFAR-100~\cite{krizhevsky2009learning}, and Tiny-ImageNet \cite{le2015tiny}. CIFAR-10 and CIFAR-100 datasets contain 50,000 training images and 10,000 test images, and they have 10 and 100 classes, respectively. GTSRB has 43 classes of different traffic signs. It has 39,209 training images and 12,630 test images in total. 
Tiny-ImageNet is a 200-class natural image dataset sub-sampled from ImageNet dataset and it contains 100,000 training and 10,000 validation images.

\vspace{5pt}
\noindent
\textbf{Models.}
Table~\ref{table:model} shows the task DNN used by each dataset. For CIFAR-10, the task DNN model is a ResNet-152 with an accuracy of 95.16\%. For GTSRB, the task DNN we use is a ResNet-34 model with an accuracy of 98.6\%. 
The ResNet-152 for CIFAR-100 has an accuracy of 80.11\%. The task model for Tiny-ImageNet is WideResNet-101, and its accuracy is 85.20\%. Please note that for Tiny-ImageNet, we use pre-trained weights on the ImageNet dataset and fine-tune them to obtain high accuracy.
We use Pytorch profiling tool "thop"\footnote{https://github.com/Lyken17/pytorch-OpCounter} to quantify model GFLOPs and parameters.
We quantize all DNN parameters into 8-bit integers (INT-8) following a uniform affine quantitizer~\cite{krishnamoorthi2018quantizing}.

\vspace{5pt}
\noindent
\textbf{Fault Model.}
In our experiment, we use two types of fault injection: random fault injection and BitFlip Attack (BFA) \cite{rakin2019bit}. For random fault injection, in each simulation run, we randomly flip $n$ bits in the model and pass one randomly selected image to the DNN model for inference. 
BFA proposed in~\cite{rakin2019bit} is the state-of-art fault injection attack on DNN models. It can crash the DNN system by injecting a small number of bit-flips by searching the most vulnerable bit progressively. A failure occurs when the predicted label is different from the one obtained in the fault-free case. 

\begin{table}[t]
\caption{Datasets and Task Models.}
\resizebox{\columnwidth}{!}{%

\begin{tabular}{|c|c|c|c|c|c|}
\hline
Dataset         & \#Classes  & Task Model     & Accuracy       & FLOPs  & Parameter \\ \hline
CIFAR-10        & 10        & ResNet-152    & 96.15\%         & 3.75 G   & 58.22 M          \\ \hline
GTSRB           & 43        & ResNet-34     & 99.40\%          & 1.16 G   & 21.30 M         \\ \hline
CIFAR-100       & 100       & ResNet-152    & 80.11\%         & 3.75 G   & 58.22 M        \\ \hline
Tiny-ImageNet   & 200       & WideResNet-101  & 85.20\%         & 22.84 G   & 126.89 M        \\ \hline
\end{tabular}

}
\label{table:model}
\end{table}

\vspace{5pt}
\noindent
\textbf{Risk Impact Martrix.}
In practice, the risk impact matrix should be determined by system designers after conducting application-specific risk analysis. One practical solution would be categorizing the risk impact into a few risk levels and filing the matrix accordingly. 
In our experiments. we simulate the impact matrix with two configurations.
\begin{itemize}
    \item \textbf{Uniform Impact}, where all entries are ones. It represents that the risk impact among all classes are equal. When the risks of different classes do not have significant differences, the uniform impact matrix can be used for simplicity.
    \item \textbf{Non-uniform Impact}, where the risk impact values are set to two different levels. 
    As classes with the low precision are not trustworthy themselves and hence have low risk, in this configuration, we assign those classes with the lowest 25\% precision with impact 1 and the others with 100.
\end{itemize}
    


\noindent
\textbf{Risk Probability Matrix.}
We obtain the risk probability matrix and the failure coverage of the checker DNN with random fault injection experiments. Figure~\ref{fig:vultimes} shows the risk probability matrices of different state-of-art model architectures trained on CIFAR-10.
We perform 4 million fault injections to obtain the result in this figure\footnote{More fault injections are performed, and the results are similar.}. For visualization purposes, the probability matrix is divided by the maximum element. The sum of all elements in the original probability matrix is $1$.

We can observe that the risk probability matrix is more task-related than model-related. 
The probability distributions are very similar across different DNN models on the same CIFAR-10 task. For example, in Figure~\ref{fig:vultimes} (a), the value between dog and cat is the highest one with fault injections. It is also true for Figure~\ref{fig:vultimes} (b) and (c).
This matrix will be used in the task simplification process.

\subsection{Effectiveness of Architecture Exploration}

\label{subsec:effectArchi}


\begin{figure*}[t]
    \centering
    \includegraphics[width=\linewidth]{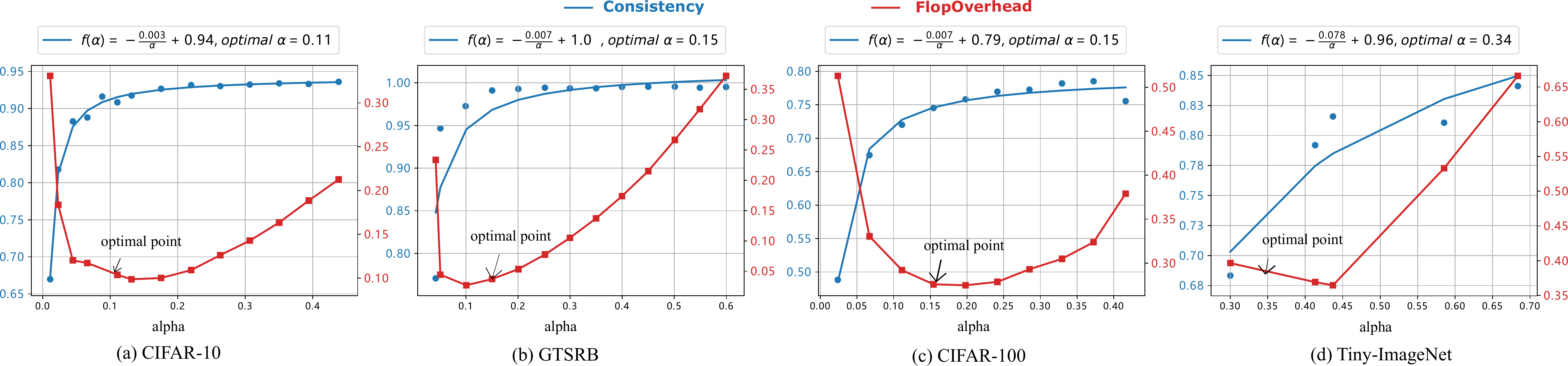}
    \caption{Consistency v.s. overhead in the architecture exploration stage.}
    \label{fig:consistency}
\end{figure*}

In this part, we study the effectiveness of the architecture exploration. 


Figure~\ref{fig:consistency} shows the consistency, indicated by blue points and approximated by blue curve, and the computational overhead, indicated by red points and connected by the red curve, of checker DNN models with different sizes trained by DeepDyve. Please note that we only investigate five checker model sizes for Tiny-ImageNet due to high training effort for this dataset, including the pre-training on ImageNet dataset and fine-tuning on Tiny-Imagenet. For all the four datasets, the consistency between the task and checker models improves with the increasing size of checker DNNs. Also, as we can observe, there is a turning point on the computational overhead curve. Before the turning point, the re-computation dominates the computational overhead O(C), and after which, the checker models' computational cost dominates the O(C).

First, compared to the architecture with the highest consistency (i.e., duplication), the optimized architecture can greatly reduce the overhead. For CIFAR-10, GTSRB, CIFAR-100, and Tiny-ImageNet, 91.82\%, 97.32\%, 72.52\% and 63.09\% overhead can be saved with 8.38\%, 1.25\%, 24.18\% and 20.81\% consistency degradation, respectively (see Table~\ref{table:tasksimplification}). 
We also observe that the consistency values vary for different datasets. 
After manual checking, we found the consistency values of the optimized architecture on the training set are almost 100\%, but it generalizes differently during inference for the four data sets.

Second, the relation between consistencies, architecture sizes, and the optimal point is well captured by THEOREM~\ref{theorem:optimal} .
For example, the resulting consistency function for CIFAR-10 is $f(\alpha) = -\frac{0.003}{\alpha} + 0.94$, and hence the optimal point is when $\alpha = 0.11$ (recall that the optimal point is $\sqrt[3]{\frac{a}{x}}$). This is compatible with the red curve where $\alpha = 0.11$ almost renders the minimal computational overhead. Similarly, the consistency function for GTSRB is $f(\alpha) = -\frac{0.007}{\alpha} + 1.0$ and the calculated optimal point is $\alpha = 0.15$, which is also compatible with the red curve. Therefore, 
the initial checker DNN architecture can be efficiently found by the proposed method.

\subsection{Effectiveness of Task Exploration}

\label{subsec:effectTask}

\begin{table}[]
\caption{Task simplification further shrink the overhead.}
\label{table:tasksimplification}

\resizebox{\columnwidth}{!}{
\begin{tabular}{|c|c|c|c|c|c|c|c|}
\hline
\multirow{2}{*}{Dataset}       & \multirow{2}{*}{Impact Matrix} & \multirow{2}{*}{\begin{tabular}[c]{@{}c@{}}Start\\ Consistency\end{tabular}} & \multicolumn{2}{c|}{Before TaskSim.} & \multicolumn{2}{c|}{After TaskSim} & \multirow{2}{*}{k} \\ \cline{4-7}
                               &                                &                                                                              & O(C)              & Wcov             & O(C)             & Wcov            &                    \\ \hline
\multirow{2}{*}{CIFAR-10}      & non-uniform                    & \multirow{2}{*}{91.62\%}                                                     & 9.18\%            & 86.94\%          & 6.88\%           & 86.12\%         & 8                  \\ \cline{2-2} \cline{4-8} 
                               & uniform                        &                                                                              & 9.18\%            & 75.90\%          & 9.11\%           & 75.90\%         & 9                  \\ \hline
\multirow{2}{*}{GTSRB}         & non-uniform                    & \multirow{2}{*}{98.75\%}                                                     & 2.68\%            & 98.46\%          & 1.94\%           & 98.23\%         & 23                 \\ \cline{2-2} \cline{4-8} 
                               & uniform                        &                                                                              & 2.68\%            & 98.15\%          & 2.56\%           & 98.15\%         & 33                 \\ \hline
\multirow{2}{*}{CIFAR-100}     & non-uniform                    & \multirow{2}{*}{75.82\%}                                                     & 27.48\%           & 67.29\%          & 24.02\%          & 66.92\%         & 83                 \\ \cline{2-2} \cline{4-8} 
                               & uniform                        &                                                                              & 27.48\%           & 74.33\%          & 27.42\%          & 74.33\%         & 99                 \\ \hline
\multirow{2}{*}{Tiny-ImageNet} & non-uniform                    & \multirow{2}{*}{79.19\%}                                                     & 36.91\%           & 76.40\%          & 35.56\%          & 75.19\%         & 186                \\ \cline{2-2} \cline{4-8} 
                               & uniform                        &                                                                              & 36.91\%           & 78.03\%          & 36.84\%          & 78.02\%         & 198                \\ \hline
\end{tabular}
}
\end{table}

\begin{figure*}[t]
    \centering
    \includegraphics[width=\linewidth]{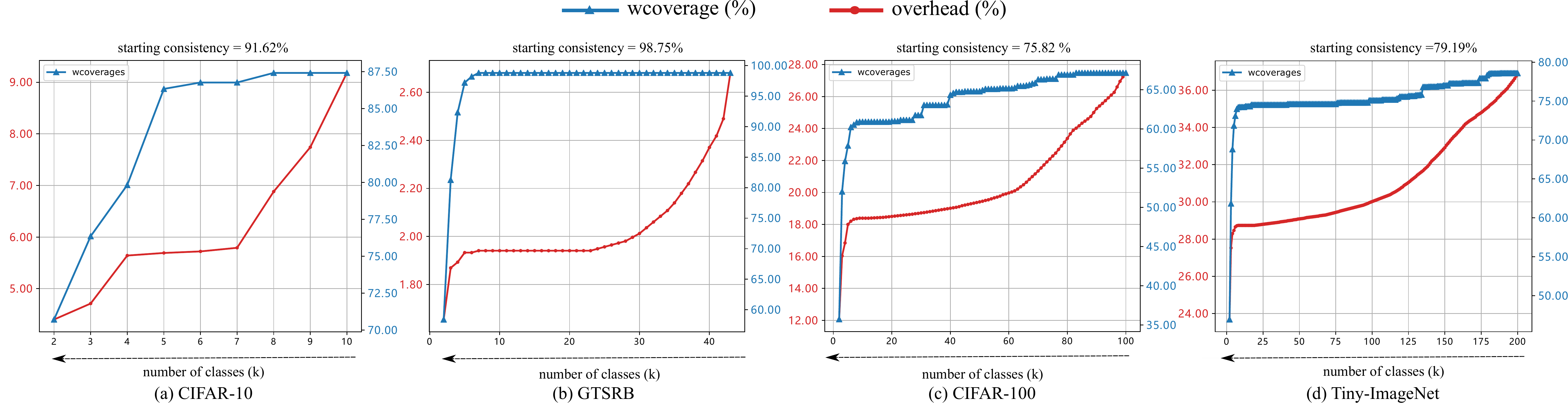}
    \caption{Tracing of the fault coverage and overhead change in task simplification process.}
    \label{fig:tasksim}
\end{figure*}

In this part, we first evaluate the weighted coverage $Wcov.$ and overhead $O(C)$ under various simplified classes $k$. 
Then, we study how the risk impact matrix affects task simplification.

First, we observe that task simplification can significantly reduce the overhead of DeepDyve with little coverage degradation. In Figure~\ref{fig:tasksim}, we use the checker DNN design obtained from Section \ref{subsec:effectArchi} and trace the weighted coverage $Wcov.$ and overhead $O(C)$ during the agglomerative class clustering process. Table~\ref{table:tasksimplification} shows the final task simplification results.
Through task simplification, we can save the overhead by $(9.18\%-6.88\%)/9.18\%=25.05\%$, 27.61\%, 12.60\%, 3.38\% for CIFAR-10, GTSRB, CIFAR-100, and Tiny-ImageNet, at the cost of 0.9\%, 0.2\%, 0.5\%, 1.6\% coverage degradation, respectively. 

\begin{table*}[t]

\caption{Comparison between DeepDyve and threshold checking.}
\label{table:comparison}
\resizebox{2\columnwidth}{!}{%
\begin{tabular}{|c|c|c|c|c|c|c|c|c|c|c|c|c|c|c|c|c|c|}
\hline
\multirow{3}{*}{Dataset}       & \multirow{3}{*}{Impact Matrix} & \multicolumn{8}{c|}{Random Fault Attack}                                                   & \multicolumn{8}{c|}{BFA}                                                                    \\ \cline{3-18} 
                               &                                & \multicolumn{4}{c|}{Threshold Checking}    & \multicolumn{4}{c|}{DeepDyve}                 & \multicolumn{4}{c|}{Threshold Checking}      & \multicolumn{4}{c|}{DeepDyve}                \\ \cline{3-18} 
                               &                                & FPR    & FNR     & O(C) & Wcov.            & FPR    & FNR     & O(C)    & Wcov.            & FPR    & FNR      & O(C) & Wcov.             & FPR    & FNR    & O(C)    & Wcov.            \\ \hline
\multirow{2}{*}{CIFAR-10}      & non-uniform                    & 0.04\% & 96.24\% & -    & \textbf{4.10\%}  & 0.00\% & 43.29\% & 6.88\%  & \textbf{82.98\%} & 0.04\% & 33.45\%  & -    & \textbf{66.39\%}  & 0.00\% & 1.01\% & 6.88\%  & \textbf{98.93\%} \\ \cline{2-18} 
                               & uniform                        & 0.04\% & 90.87\% & -     & \textbf{9.12\%}  & 0.00\% & 24.06\%  & 9.11\%  & \textbf{75.94\%} & 0.04\% & 55.52\%  & -    & \textbf{44.48\%}  & 0.00\% & 2.15\% & 9.11\%  & \textbf{97.85\%} \\ \hline
\multirow{2}{*}{GTSRB}         & non-uniform                    & 0.00\% & 99.72\% & -    & \textbf{0.3\%}   & 0.00\% & 16.53\% & 1.94\%  & \textbf{95.49\%} & 0.00\% & 100.00\% & -    & \textbf{0.00\%}   & 0.00\% & 0.90\% & 1.94\%  & \textbf{99.75\%} \\ \cline{2-18} 
                               & uniform                        & 0.00\% & 96.85\% & -     & \textbf{3.15\%}  & 0.00\% & 5.07\%  & 2.56\%  & \textbf{94.93\%} & 0.00\% & 99.88\%  & -    & \textbf{0.12\%}   & 0.00\% & 0.79\% & 2.56\%  & \textbf{99.21\%} \\ \hline
\multirow{2}{*}{CIFAR-100}     & non-uniform                    & 0.02\% & 65.11\% & -    & \textbf{33.48\%} & 0.00\% & 31.51\% & 24.02\% & \textbf{74.67\%} & 0.02\% & 14.50\%  & -    & \textbf{86.61\%}  & 0.00\% & 1.84\% & 24.02\% & \textbf{99.50\%} \\ \cline{2-18} 
                               & uniform                        & 0.02\% & 86.99\% & -     & \textbf{13.01\%} & 0.00\% & 19.21\% & 27.42\% & \textbf{80.79\%} & 0.02\% & 1.05\%   & -    & \textbf{98.95\%}  & 0.00\% & 0.37\% & 27.42\% & \textbf{99.63\%} \\ \hline
\multirow{2}{*}{Tiny-ImageNet} & non-uniform                    & 0.02\% & 69.31\% & -    & \textbf{30.56\%} & 0.00\% & 19.46\% & 35.19\% & \textbf{82.00\%} & 0.02\% & 0.00\%   & -    & \textbf{100.00\%} & 0.00\% & 0.05\% & 35.19\% & \textbf{99.94\%} \\ \cline{2-18} 
                               & uniform                        & 0.06\% & 87.34\% & -    & \textbf{12.65\%} & 0.00\% & 17.02\% & 36.84\% & \textbf{82.98\%} & 0.06\% & 0.08\%   & -    & \textbf{99.92\%}  & 0.00\% & 0.02\% & 36.84\% & \textbf{99.98\%} \\ \hline
\end{tabular}
}

\end{table*}

We can also observe that fault coverage and overhead vary a lot under different impact matrix configurations. First, the initial weighted coverage values before task simplification is different for uniform and non-uniform settings, because the $I_{ij}$ term in the definition of weighed coverage given by Equation~\ref{equ:weightedCoverage} varies.
Second, we observe much more overhead savings can be achieved in the non-uniform case. For example, the overhead saving can be improved from 0.76\% to 25.05\% on CIFAR-10 when changing to
non-uniform impact matrix. This is because, the impact of classes with low precision is set as lower values under such circumstances, which provides more opportunities for task simplification.  
In other words, a reasonable impact matrix is beneficial for protection with DeepDyve and hence is highly recommended. 

\subsection{DeepDyve vs. Threshold Checking}
\label{subsec:comparisonthresh}

In this section, we compare the performance of DeepDyve with the Threshold Checking scheme proposed in~\cite{li2017understanding}. We experiment under both random fault attack and BitFlip Attack (BFA) settings. We show the results in Table~\ref{table:comparison}, including the false-positive rate (FPR), false-negative rate (FNR), computational overhead (O(C)), and the weighted coverage (Wcov.). 

First, we observe DeepDyve significantly outperforms Threshold Checking in terms of Wcov., which in turn leads to smaller FNR. As discussed in Section~\ref{sec:backgorund}, most intermediate activation values locate in the normal range even under fault attack, especially for the quantized DNN case. Hence, most faults are missed with Threshold Checking, but they can be detected by DeepDyve. 

Second, in most cases, we can observe both DeepDyve and Threshold Checking performs better under BFA compared to random fault attacks. We manually check the internal values of DNN under fault attack and find that the magnitude of value change is larger in BFA than in random fault attacks, thereby making fault detection easier. 

Third, the FPR of Threshold Checking is above zero while that of our DeepDyve system is zero. For example, the FPR of Threshold Checking for CIFAR-10 dataset is 0.04\%. As discussed in Section \ref{sec:backgorund}, threshold Detection sets the thresholds as 1.1 times the maximum and a minimum of each layer's normal activation values on the training set. On the testing set, there are few exceptions where the activation values are beyond this range. In contrast, in the normal execution of DeepDyve, the comparator's false positives, which are caused by inconsistencies between task and checker models, are subject to re-computation, and hence, the system's false positives are guaranteed to be zero.



\subsection{Impact of Model Accuracy}

\label{subsec:casestudy}
Previous experiments suggest that the overhead of DeepDyve for CIFAR-100 and Tiny-ImageNet dataset are quite high, even after task simplification. 
This is because the per-class accuracy of the task model on these two datasets varies and some of them are very low, as shown in Figure~\ref{fig:perclassacc}. In safety-critical applications, a class is deserved to be protected only when its accuracy is high enough. Considering the above, we conduct a case study on CIFAR-100 and let the impact of 75\% of the classes with the lowest precision to be zero. 

\begin{figure}[ht]
    \centering
    \includegraphics[width=0.8\linewidth]{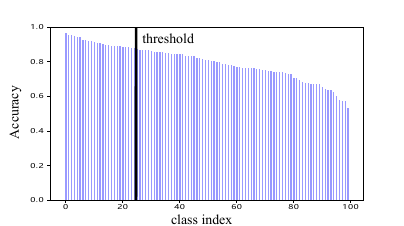}
    \caption{Per-class accuracy for CIFAR-100.}
    \label{fig:perclassacc}
\end{figure}

Previously, the computational overhead induced by the checker DNN before task simplification was 27.48\% (See Table \ref{table:tasksimplification}). With the above setting, as we only care classes with non-zero impact, we let the comparator only check the inconsistencies of these classes. Given this, the overhead induced by the checker DNN before task simplification is 16.17\%. Task simplification can further reduce the overhead from 16.17\% to 9.88\% without loss of weighted coverage (38.89\% overhead savings). Also, the simplified model can reach 90.66\% weighted coverage under random fault attack and 99.78\% weighted coverage under BFA attack. 

\section{Discussions}
\label{sec:discussion}

In this section, we discuss the robustness of the proposed DeepDyve architecture and its limitations. 

\subsection{Robustness of DeepDyve}
\label{sec:checkerFaults}



\noindent
\textbf{What if the checker DNN has faults?}
If the checker DNN is faulty while the task DNN is correct, the final inference accuracy would remain the same, because the system would accept the output from the task model anyway. There could be extra latency. To mitigate this issue, we could leverage various hardening techniques (e.g., secure enclaves) to protect it at a reasonable cost since the checker DNN is much smaller than the task DNN model.

\vspace{5pt}
\noindent
\textbf{Attack on DeepDyve.}
Attackers need to create consistent faulty outputs to bypass the comparison logic of DeepDyve to successful launch their attacks. 

One way to achieve this is to inject faults into the task and checker models simultaneously.
However, the cost of launching such an attack in practice is very high, if not impossible. On the one hand, simultaneously injecting faults at two specific positions is difficult. For example, row-hammer attack relies on the weakness of physical memory row, and it cannot be fully controlled. On the other hand, if low-precision fault injection technique is used to inject random faults into the two DNNs, the probability distribution of the faulty output of a DNN is given by its risk probability matrix, and the probability of two DNNs' outputs happens to be the same is given by
\begin{equation}
    P_{collision} = \sum_i^N{p_i*q_i},
\end{equation}
where $p_i$ and $q_i$ are the probabilities of task DNN and checker DNN generating the same output $i$, respectively, and $N$ is the number of classes. Obviously, this value decreases with the increase of classes. The $P_{collision}$ values are 8.9\%, 2.77\%, 0.87\%, and 0.49\% for CIFAR-10, GTSRB, CIFAR-100, and Tiny-ImageNet, respectively.
Given that the probability of DNNs generating wrong outputs under random faults $P_{error}$ is extremely low. The possibility for such attack to succeed is negligible.

Another way to successfully launch fault injection attacks is to target those inconsistent cases and make them consistent with faulty result. 
To achieve this objective, however, attackers need to be able to tell whether an incoming data is consistent at runtime and perform fault injection before its inference is finished. This is a daunting objective to achieve, especially considering the usual long preparation time for fault injection.


\subsection{Limitations and Future Work}
\label{sec:limitation}

DeepDyve brings latency and reduces throughput due to recomputation, which needs to be considered when performing real-time tasks. For example, there is about 1\% of throughput loss for the above DeepDyve model constructed on the GTSRB dataset. This overhead can be reduced when the consistency rate between the checker DNN and the task DNN models increases. 

Also, our current evaluation is only on the 8-bit integer (INT-8) data representation. One reason is that INT-8 is a popular data type and supported by well known deep learning frameworks and tool-chains, such as Pytorh~\footnote{https://pytorch.org/docs/stable/quantization.html}, Tensorflow\footnote{https://www.tensorflow.org/lite/performance/post\_training\_quantization}, NVIDIA\textregistered TensorRD~\footnote{https://developer.nvidia.com/tensorrt}, and Xilinx\textregistered DNNDK~\footnote{https://www.xilinx.com/products/design-tools/ai-inference/edge-ai-platform.html\#dnndk}. Popular hardware platforms like TPU~\cite{jouppi2017datacenter} support integer operations as well. The second reason is that the Pytorch 1.3 (the one we use) and the code from BFA attack~\footnote{https://github.com/elliothe/BFA} only supports INT-8 quantization in its current version. We shall extend our work to study the impact of different data types in the future.

\section{Related Works}
\label{sec:related}




DNNs have become an enabling technology for many safety-critical
applications, and hence there is a growing research interest for
fault-tolerant DNN designs.

In the literature, some works focused on improving fault-tolerance of DNN systems built with unreliable \textit{resistive random-access memory} (RRAM)~\cite{chen2017accelerator, xia2017fault, liu2017rescuing}, an emerging device with huge energy efficiency benefits. However, this technology is still in its infancy and is not applicable for safety-critical AI applications. Existing fault-tolerant solutions for DNN systems built with conventional CMOS devices can be categorized as \textit{passive} solutions or \textit{active} solutions. 

Passive fault-tolerant solutions reduce the probability of failure occurrences for a certain level at the design stage, including \emph{fault-mitigation} techniques for the DNN design itself and \emph{fault-masking} solutions with redundant circuitries. To be specific, fault-mitigation techniques include: (i) \textit{Resilient training}~\cite{deng2015retraining, kim2018matic, yang2017sram, He2020cvpr}, which explicitly considers hardware faults during the DNN training phase; (ii) \textit{Resilient architecture design}~\cite{schorn2019automated, chiu1993robustness, chu1990fault, dias2010ftset}, which tries to find an error-resilient network architecture or inserts redundancies for those critical elements in the system; (iii) \textit{Device hardening}~\cite{li2017understanding, azizimazreah2018tolerating}, which tries to protect some critical storage elements in the system by selectively hardening them. The above solutions require significant design effort, and they usually cannot provide sufficient fault coverage. Fault-masking techniques are conceptually simple. We could use \textit{error-correcting codes} (ECC) to protect memory elements~\cite{reagen2016minerva} and \textit{triple-modular redundancy} (TMR) to protect computational units~\cite{yan2020single, mahdiani2012relaxed}. However, the corresponding hardware overhead is exceptionally high. 

Active fault-tolerant solutions detect and recover from faults by re-organizing the system in real-time. To achieve active fault-tolerance, we need to be able to detect faults online. There are a few solutions in this direction~\cite{li2019d2nn,li2017understanding,schorn2018efficient}. Among these techniques, \cite{li2017understanding,schorn2018efficient} has small hardware overhead, but their fault detection capabilities are limited, especially for quantized DNNs; the fault coverage of the replication-based solution in~\cite{li2019d2nn} is comparable to DeepDyve, but its overhead is much higher.  


\section{Conclusion}
\label{sec:conclusion}

In this paper, we proposed DeepDyve, a novel dynamic verification technique against fault injection attacks on DNN systems. By introducing a far simpler and smaller checker DNN into the system, DeepDyve significantly outperforms state-of-the-art solutions, reducing 90\% risks at around 10\% computational overhead. 


\begin{acks}
This work is supported in part by \grantsponsor{}{General Research Fund (GRF)}{} of Hong Kong Research Grants Council (RGC) under Grant No. \grantnum[]{}{14205018} and No. \grantnum[]{}{14205420}, and in part by \grantsponsor{}{National Natural Science Foundation of China (NSFC)}{} under Grant No. \grantnum[]{}{61532017}.
\end{acks}

\balance
\bibliographystyle{ACM-Reference-Format}
\bibliography{ref.bib}


\begin{thebibliography}{46}


\ifx \showCODEN    \undefined \def \showCODEN     #1{\unskip}     \fi
\ifx \showDOI      \undefined \def \showDOI       #1{#1}\fi
\ifx \showISBNx    \undefined \def \showISBNx     #1{\unskip}     \fi
\ifx \showISBNxiii \undefined \def \showISBNxiii  #1{\unskip}     \fi
\ifx \showISSN     \undefined \def \showISSN      #1{\unskip}     \fi
\ifx \showLCCN     \undefined \def \showLCCN      #1{\unskip}     \fi
\ifx \shownote     \undefined \def \shownote      #1{#1}          \fi
\ifx \showarticletitle \undefined \def \showarticletitle #1{#1}   \fi
\ifx \showURL      \undefined \def \showURL       {\relax}        \fi
\providecommand\bibfield[2]{#2}
\providecommand\bibinfo[2]{#2}
\providecommand\natexlab[1]{#1}
\providecommand\showeprint[2][]{arXiv:#2}

\bibitem[\protect\citeauthoryear{Austin}{Austin}{1999}]%
        {austin1999diva}
\bibfield{author}{\bibinfo{person}{Todd~M Austin}.}
  \bibinfo{year}{1999}\natexlab{}.
\newblock \showarticletitle{DIVA: A reliable substrate for deep submicron
  microarchitecture design}. In \bibinfo{booktitle}{\emph{Proceedings of the
  32nd Annual ACM/IEEE International Symposium on Microarchitecture (MICRO)}}.
  \bibinfo{publisher}{IEEE}, \bibinfo{pages}{196--207}.
\newblock


\bibitem[\protect\citeauthoryear{Azizimazreah, Gu, Gu, and Chen}{Azizimazreah
  et~al\mbox{.}}{2018}]%
        {azizimazreah2018tolerating}
\bibfield{author}{\bibinfo{person}{Arash Azizimazreah},
  \bibinfo{person}{Yongbin Gu}, \bibinfo{person}{Xiang Gu}, {and}
  \bibinfo{person}{Lizhong Chen}.} \bibinfo{year}{2018}\natexlab{}.
\newblock \showarticletitle{Tolerating soft errors in deep learning
  accelerators with reliable on-chip memory designs}. In
  \bibinfo{booktitle}{\emph{IEEE International Conference on Networking,
  Architecture and Storage (NAS)}}. \bibinfo{publisher}{IEEE},
  \bibinfo{pages}{1--10}.
\newblock


\bibitem[\protect\citeauthoryear{Chen, Li, Chen, Deng, Shen, Liang, and
  Jiang}{Chen et~al\mbox{.}}{2017}]%
        {chen2017accelerator}
\bibfield{author}{\bibinfo{person}{Lerong Chen}, \bibinfo{person}{Jiawen Li},
  \bibinfo{person}{Yiran Chen}, \bibinfo{person}{Qiuping Deng},
  \bibinfo{person}{Jiyuan Shen}, \bibinfo{person}{Xiaoyao Liang}, {and}
  \bibinfo{person}{Li Jiang}.} \bibinfo{year}{2017}\natexlab{}.
\newblock \showarticletitle{Accelerator-friendly neural-network training:
  Learning variations and defects in RRAM crossbar}. In
  \bibinfo{booktitle}{\emph{Proceedings of the Conference on Design, Automation
  \& Test in Europe (DATE)}}. \bibinfo{publisher}{European Design and
  Automation Association}, \bibinfo{pages}{19--24}.
\newblock


\bibitem[\protect\citeauthoryear{Chen, Krishna, Emer, and Sze}{Chen
  et~al\mbox{.}}{2016}]%
        {chen2016eyeriss}
\bibfield{author}{\bibinfo{person}{Yu-Hsin Chen}, \bibinfo{person}{Tushar
  Krishna}, \bibinfo{person}{Joel~S Emer}, {and} \bibinfo{person}{Vivienne
  Sze}.} \bibinfo{year}{2016}\natexlab{}.
\newblock \showarticletitle{Eyeriss: An energy-efficient reconfigurable
  accelerator for deep convolutional neural networks}. In
  \bibinfo{booktitle}{\emph{IEEE Journal of Solid-State Circuits (JSSC)}}.
  \bibinfo{publisher}{IEEE}, \bibinfo{pages}{127--138}.
\newblock


\bibitem[\protect\citeauthoryear{Chiu, Mehrotra, Mohan, and Ranka}{Chiu
  et~al\mbox{.}}{1993}]%
        {chiu1993robustness}
\bibfield{author}{\bibinfo{person}{C-T Chiu}, \bibinfo{person}{Kishan
  Mehrotra}, \bibinfo{person}{Chilukuri~K Mohan}, {and} \bibinfo{person}{Sanjay
  Ranka}.} \bibinfo{year}{1993}\natexlab{}.
\newblock \showarticletitle{Robustness of feedforward neural networks}. In
  \bibinfo{booktitle}{\emph{IEEE International Conference on Neural Networks
  (ICNN)}}. \bibinfo{publisher}{IEEE}, \bibinfo{pages}{783--788}.
\newblock


\bibitem[\protect\citeauthoryear{Chu and Wah}{Chu and Wah}{1990}]%
        {chu1990fault}
\bibfield{author}{\bibinfo{person}{L-C Chu} {and} \bibinfo{person}{Benjamin~W
  Wah}.} \bibinfo{year}{1990}\natexlab{}.
\newblock \showarticletitle{Fault tolerant neural networks with hybrid
  redundancy}. In \bibinfo{booktitle}{\emph{International Joint Conference on
  Neural Networks (IJCNN)}}. \bibinfo{publisher}{IEEE},
  \bibinfo{pages}{639--649}.
\newblock


\bibitem[\protect\citeauthoryear{Deng, Rang, Du, Wang, Li, Temam, Ienne, Novo,
  Li, Chen, et~al\mbox{.}}{Deng et~al\mbox{.}}{2015}]%
        {deng2015retraining}
\bibfield{author}{\bibinfo{person}{Jiacnao Deng}, \bibinfo{person}{Yuntan
  Rang}, \bibinfo{person}{Zidong Du}, \bibinfo{person}{Ymg Wang},
  \bibinfo{person}{Huawei Li}, \bibinfo{person}{Olivier Temam},
  \bibinfo{person}{Paolo Ienne}, \bibinfo{person}{David Novo},
  \bibinfo{person}{Xiaowei Li}, \bibinfo{person}{Yunji Chen}, {et~al\mbox{.}}}
  \bibinfo{year}{2015}\natexlab{}.
\newblock \showarticletitle{Retraining-based timing error mitigation for
  hardware neural networks}. In \bibinfo{booktitle}{\emph{Design, Automation \&
  Test in Europe Conference \& Exhibition (DATE)}}. \bibinfo{publisher}{IEEE},
  \bibinfo{pages}{593--596}.
\newblock


\bibitem[\protect\citeauthoryear{Dias, Borralho, Fontes, and Antunes}{Dias
  et~al\mbox{.}}{2010}]%
        {dias2010ftset}
\bibfield{author}{\bibinfo{person}{Fernando~Morgado Dias}, \bibinfo{person}{Rui
  Borralho}, \bibinfo{person}{Pedro Fontes}, {and} \bibinfo{person}{Ana
  Antunes}.} \bibinfo{year}{2010}\natexlab{}.
\newblock \showarticletitle{FTSET-a software tool for fault tolerance
  evaluation and improvement}. In \bibinfo{booktitle}{\emph{Neural Computing
  and Applications (Neural. Comput. Appl.)}}, Vol.~\bibinfo{volume}{19}.
  \bibinfo{publisher}{Springer}, \bibinfo{pages}{701--712}.
\newblock


\bibitem[\protect\citeauthoryear{Goodfellow, Shlens, and Szegedy}{Goodfellow
  et~al\mbox{.}}{2015}]%
        {GoodfellowSS14}
\bibfield{author}{\bibinfo{person}{Ian~J. Goodfellow},
  \bibinfo{person}{Jonathon Shlens}, {and} \bibinfo{person}{Christian
  Szegedy}.} \bibinfo{year}{2015}\natexlab{}.
\newblock \showarticletitle{Explaining and Harnessing Adversarial Examples}. In
  \bibinfo{booktitle}{\emph{3rd International Conference on Learning
  Representations (ICLR)}}. \bibinfo{publisher}{OpenReview.net}.
\newblock


\bibitem[\protect\citeauthoryear{Han, Mao, and Dally}{Han
  et~al\mbox{.}}{2016}]%
        {han2015deep}
\bibfield{author}{\bibinfo{person}{Song Han}, \bibinfo{person}{Huizi Mao},
  {and} \bibinfo{person}{William~J. Dally}.} \bibinfo{year}{2016}\natexlab{}.
\newblock \showarticletitle{Deep Compression: Compressing Deep Neural Network
  with Pruning, Trained Quantization and Huffman Coding}. In
  \bibinfo{booktitle}{\emph{4th International Conference on Learning
  Representations (ICLR)}}. \bibinfo{publisher}{OpenReview.net}.
\newblock


\bibitem[\protect\citeauthoryear{Hinton, Vinyals, and Dean}{Hinton
  et~al\mbox{.}}{2014}]%
        {hinton2015distilling}
\bibfield{author}{\bibinfo{person}{Geoffrey Hinton}, \bibinfo{person}{Oriol
  Vinyals}, {and} \bibinfo{person}{Jeffrey Dean}.}
  \bibinfo{year}{2014}\natexlab{}.
\newblock \showarticletitle{Distilling the Knowledge in a Neural Network}. In
  \bibinfo{booktitle}{\emph{NIPS Deep Learning and Representation Learning
  Workshop (NIPS Workshop)}}. \bibinfo{publisher}{Curran Associates Inc.}
\newblock


\bibitem[\protect\citeauthoryear{Hong, Frigo, Kaya, Giuffrida, and
  Dumitraș}{Hong et~al\mbox{.}}{2019}]%
        {hong2019terminal}
\bibfield{author}{\bibinfo{person}{Sanghyun Hong}, \bibinfo{person}{Pietro
  Frigo}, \bibinfo{person}{Yi{\u{g}}itcan Kaya}, \bibinfo{person}{Cristiano
  Giuffrida}, {and} \bibinfo{person}{Tudor Dumitraș}.}
  \bibinfo{year}{2019}\natexlab{}.
\newblock \showarticletitle{Terminal brain damage: Exposing the graceless
  degradation in deep neural networks under hardware fault attacks}. In
  \bibinfo{booktitle}{\emph{28th $\{$USENIX$\}$ Security Symposium (USENIX
  Security)}}. \bibinfo{publisher}{USENIX Association},
  \bibinfo{pages}{497--514}.
\newblock


\bibitem[\protect\citeauthoryear{Howard, Zhu, Chen, Kalenichenko, Wang, Weyand,
  Andreetto, and Adam}{Howard et~al\mbox{.}}{2017}]%
        {howard2017mobilenets}
\bibfield{author}{\bibinfo{person}{Andrew~G Howard}, \bibinfo{person}{Menglong
  Zhu}, \bibinfo{person}{Bo Chen}, \bibinfo{person}{Dmitry Kalenichenko},
  \bibinfo{person}{Weijun Wang}, \bibinfo{person}{Tobias Weyand},
  \bibinfo{person}{Marco Andreetto}, {and} \bibinfo{person}{Hartwig Adam}.}
  \bibinfo{year}{2017}\natexlab{}.
\newblock \showarticletitle{Mobilenets: Efficient convolutional neural networks
  for mobile vision applications}.
\newblock \bibinfo{journal}{\emph{arXiv preprint arXiv:1704.04861}}
  (\bibinfo{year}{2017}).
\newblock


\bibitem[\protect\citeauthoryear{ISO}{ISO}{2016}]%
        {iso26262}
\bibfield{author}{\bibinfo{person}{ISO}.} \bibinfo{year}{{2016}}\natexlab{}.
\newblock \showarticletitle{{ISO-26262: Road vehicles -- Functional safety}}.
  \bibinfo{publisher}{{ISO, Geneva, Switzerland}}.
\newblock


\bibitem[\protect\citeauthoryear{Jouppi and et~al.}{Jouppi and et~al.}{2017}]%
        {jouppi2017datacenter}
\bibfield{author}{\bibinfo{person}{Norman~P Jouppi} {and} \bibinfo{person}{et
  al.}} \bibinfo{year}{2017}\natexlab{}.
\newblock \showarticletitle{In-datacenter performance analysis of a tensor
  processing unit}. In \bibinfo{booktitle}{\emph{ACM/IEEE 44th Annual
  International Symposium on Computer Architecture (ISCA)}}.
  \bibinfo{publisher}{IEEE}, \bibinfo{pages}{1--12}.
\newblock


\bibitem[\protect\citeauthoryear{Kim, Howe, Moreau, Alaghi, Ceze, and
  Sathe}{Kim et~al\mbox{.}}{2018}]%
        {kim2018matic}
\bibfield{author}{\bibinfo{person}{Sung Kim}, \bibinfo{person}{Patrick Howe},
  \bibinfo{person}{Thierry Moreau}, \bibinfo{person}{Armin Alaghi},
  \bibinfo{person}{Luis Ceze}, {and} \bibinfo{person}{Visvesh Sathe}.}
  \bibinfo{year}{2018}\natexlab{}.
\newblock \showarticletitle{MATIC: Learning around errors for efficient
  low-voltage neural network accelerators}. In
  \bibinfo{booktitle}{\emph{Design, Automation \& Test in Europe Conference \&
  Exhibition (DATE)}}. \bibinfo{publisher}{IEEE}, \bibinfo{pages}{1--6}.
\newblock


\bibitem[\protect\citeauthoryear{Kim, Ross, and et. al.}{Kim
  et~al\mbox{.}}{2014}]%
        {kim2014flipping}
\bibfield{author}{\bibinfo{person}{Yoongu Kim}, \bibinfo{person}{Daly Ross},
  {and} \bibinfo{person}{et. al.}} \bibinfo{year}{2014}\natexlab{}.
\newblock \showarticletitle{Flipping bits in memory without accessing them: An
  experimental study of DRAM disturbance errors}. In
  \bibinfo{booktitle}{\emph{{ACM/IEEE} 41st International Symposium on Computer
  Architecture (ISCA)}}. \bibinfo{publisher}{IEEE}, \bibinfo{pages}{361--372}.
\newblock


\bibitem[\protect\citeauthoryear{Krishnamoorthi}{Krishnamoorthi}{2018}]%
        {krishnamoorthi2018quantizing}
\bibfield{author}{\bibinfo{person}{Raghuraman Krishnamoorthi}.}
  \bibinfo{year}{2018}\natexlab{}.
\newblock \showarticletitle{Quantizing deep convolutional networks for
  efficient inference: A whitepaper}. In \bibinfo{booktitle}{\emph{arXiv
  preprint arXiv:1806.08342}}.
\newblock


\bibitem[\protect\citeauthoryear{Krizhevsky, Hinton, et~al\mbox{.}}{Krizhevsky
  et~al\mbox{.}}{2009}]%
        {krizhevsky2009learning}
\bibfield{author}{\bibinfo{person}{Alex Krizhevsky}, \bibinfo{person}{Geoffrey
  Hinton}, {et~al\mbox{.}}} \bibinfo{year}{2009}\natexlab{}.
\newblock \bibinfo{booktitle}{\emph{Learning multiple layers of features from
  tiny images}}.
\newblock \bibinfo{type}{{T}echnical {R}eport}.
  \bibinfo{institution}{Citeseer}.
\newblock


\bibitem[\protect\citeauthoryear{Kurakin, Goodfellow, and Bengio}{Kurakin
  et~al\mbox{.}}{2017}]%
        {KurakinGB17}
\bibfield{author}{\bibinfo{person}{Alexey Kurakin}, \bibinfo{person}{Ian~J.
  Goodfellow}, {and} \bibinfo{person}{Samy Bengio}.}
  \bibinfo{year}{2017}\natexlab{}.
\newblock \showarticletitle{Adversarial Machine Learning at Scale}. In
  \bibinfo{booktitle}{\emph{5th International Conference on Learning
  Representations (ICLR)}}. \bibinfo{publisher}{OpenReview.net}.
\newblock


\bibitem[\protect\citeauthoryear{Le and Yang}{Le and Yang}{2015}]%
        {le2015tiny}
\bibfield{author}{\bibinfo{person}{Ya Le} {and} \bibinfo{person}{Xuan Yang}.}
  \bibinfo{year}{2015}\natexlab{}.
\newblock \showarticletitle{Tiny imagenet visual recognition challenge}.
\newblock \bibinfo{journal}{\emph{Stanford CS 231N Course}}.
\newblock


\bibitem[\protect\citeauthoryear{LeCun, Denker, and Solla}{LeCun
  et~al\mbox{.}}{1990}]%
        {lecun1990optimal}
\bibfield{author}{\bibinfo{person}{Yann LeCun}, \bibinfo{person}{John~S
  Denker}, {and} \bibinfo{person}{Sara~A Solla}.}
  \bibinfo{year}{1990}\natexlab{}.
\newblock \showarticletitle{Optimal brain damage}. In
  \bibinfo{booktitle}{\emph{Advances in neural information processing systems
  (NIPS)}}. \bibinfo{publisher}{Curran Associates Inc.},
  \bibinfo{pages}{598--605}.
\newblock


\bibitem[\protect\citeauthoryear{Li and et~al.}{Li and et~al.}{2017}]%
        {li2017understanding}
\bibfield{author}{\bibinfo{person}{Guanpeng Li} {and} \bibinfo{person}{et al.}}
  \bibinfo{year}{2017}\natexlab{}.
\newblock \showarticletitle{Understanding error propagation in deep learning
  neural network (DNN) accelerators and applications}. In
  \bibinfo{booktitle}{\emph{Proceedings of the International Conference for
  High Performance Computing, Networking, Storage and Analysis (SC)}}.
  \bibinfo{publisher}{ACM}, \bibinfo{pages}{8:1--8:12}.
\newblock


\bibitem[\protect\citeauthoryear{Li, Liu, Li, Tian, Luo, and Xu}{Li
  et~al\mbox{.}}{2019}]%
        {li2019d2nn}
\bibfield{author}{\bibinfo{person}{Yu Li}, \bibinfo{person}{Yannan Liu},
  \bibinfo{person}{Min Li}, \bibinfo{person}{Ye Tian}, \bibinfo{person}{Bo
  Luo}, {and} \bibinfo{person}{Qiang Xu}.} \bibinfo{year}{2019}\natexlab{}.
\newblock \showarticletitle{D2NN: a fine-grained dual modular redundancy
  framework for deep neural networks}. In \bibinfo{booktitle}{\emph{Proceedings
  of the 35th Annual Computer Security Applications Conference (ACSAC)}}.
  \bibinfo{publisher}{ACM}, \bibinfo{pages}{138--147}.
\newblock


\bibitem[\protect\citeauthoryear{Liu, Hu, Strachan, and Li}{Liu
  et~al\mbox{.}}{2017a}]%
        {liu2017rescuing}
\bibfield{author}{\bibinfo{person}{Chenchen Liu}, \bibinfo{person}{Miao Hu},
  \bibinfo{person}{John~Paul Strachan}, {and} \bibinfo{person}{Hai Li}.}
  \bibinfo{year}{2017}\natexlab{a}.
\newblock \showarticletitle{Rescuing memristor-based neuromorphic design with
  high defects}. In \bibinfo{booktitle}{\emph{54th ACM/EDAC/IEEE Design
  Automation Conference (DAC)}}. \bibinfo{publisher}{IEEE},
  \bibinfo{pages}{1--6}.
\newblock


\bibitem[\protect\citeauthoryear{Liu, Wei, Luo, and Xu}{Liu
  et~al\mbox{.}}{2017b}]%
        {liu2017fault}
\bibfield{author}{\bibinfo{person}{Yannan Liu}, \bibinfo{person}{Lingxiao Wei},
  \bibinfo{person}{Bo Luo}, {and} \bibinfo{person}{Qiang Xu}.}
  \bibinfo{year}{2017}\natexlab{b}.
\newblock \showarticletitle{Fault injection attack on deep neural network}. In
  \bibinfo{booktitle}{\emph{IEEE/ACM International Conference on Computer-Aided
  Design (ICCAD)}}. \bibinfo{publisher}{IEEE}, \bibinfo{pages}{131--138}.
\newblock


\bibitem[\protect\citeauthoryear{Luo, Liu, Wei, and Xu}{Luo
  et~al\mbox{.}}{2018}]%
        {luo2018towards}
\bibfield{author}{\bibinfo{person}{Bo Luo}, \bibinfo{person}{Yannan Liu},
  \bibinfo{person}{Lingxiao Wei}, {and} \bibinfo{person}{Qiang Xu}.}
  \bibinfo{year}{2018}\natexlab{}.
\newblock \showarticletitle{Towards imperceptible and robust adversarial
  example attacks against neural networks}. In
  \bibinfo{booktitle}{\emph{Proceedings of the Thirty-Second AAAI Conference on
  Artificial Intelligence (AAAI)}}. \bibinfo{publisher}{AAAI Press},
  \bibinfo{pages}{1652--1659}.
\newblock


\bibitem[\protect\citeauthoryear{Madry, Makelov, Schmidt, Tsipras, and
  Vladu}{Madry et~al\mbox{.}}{2018}]%
        {MadryMSTV18}
\bibfield{author}{\bibinfo{person}{Aleksander Madry},
  \bibinfo{person}{Aleksandar Makelov}, \bibinfo{person}{Ludwig Schmidt},
  \bibinfo{person}{Dimitris Tsipras}, {and} \bibinfo{person}{Adrian Vladu}.}
  \bibinfo{year}{2018}\natexlab{}.
\newblock \showarticletitle{Towards Deep Learning Models Resistant to
  Adversarial Attacks}. In \bibinfo{booktitle}{\emph{6th International
  Conference on Learning Representations (ICLR)}}.
  \bibinfo{publisher}{OpenReview.net}.
\newblock


\bibitem[\protect\citeauthoryear{Mahdiani, Fakhraie, and Lucas}{Mahdiani
  et~al\mbox{.}}{2012}]%
        {mahdiani2012relaxed}
\bibfield{author}{\bibinfo{person}{Hamid~Reza Mahdiani},
  \bibinfo{person}{Sied~Mehdi Fakhraie}, {and} \bibinfo{person}{Caro Lucas}.}
  \bibinfo{year}{2012}\natexlab{}.
\newblock \showarticletitle{Relaxed fault-tolerant hardware implementation of
  neural networks in the presence of multiple transient errors}.
\newblock \bibinfo{journal}{\emph{IEEE transactions on neural networks and
  learning systems (TNNLS)}}  \bibinfo{volume}{23},
  \bibinfo{pages}{1215--1228}.
\newblock


\bibitem[\protect\citeauthoryear{Matsubayashi, Satoh, and Ishii}{Matsubayashi
  et~al\mbox{.}}{2016}]%
        {matsubayashi2016clock}
\bibfield{author}{\bibinfo{person}{Masato Matsubayashi},
  \bibinfo{person}{Akashi Satoh}, {and} \bibinfo{person}{Jun Ishii}.}
  \bibinfo{year}{2016}\natexlab{}.
\newblock \showarticletitle{Clock glitch generator on SAKURA-G for fault
  injection attack against a cryptographic circuit}. In
  \bibinfo{booktitle}{\emph{IEEE 5th Global Conference on Consumer Electronics
  (GCCE)}}. \bibinfo{publisher}{IEEE}, \bibinfo{pages}{1--4}.
\newblock


\bibitem[\protect\citeauthoryear{Meng and Chen}{Meng and Chen}{2017}]%
        {meng2017magnet}
\bibfield{author}{\bibinfo{person}{Dongyu Meng} {and} \bibinfo{person}{Hao
  Chen}.} \bibinfo{year}{2017}\natexlab{}.
\newblock \showarticletitle{Magnet: a two-pronged defense against adversarial
  examples}. In \bibinfo{booktitle}{\emph{Proceedings of the ACM SIGSAC
  conference on computer and communications security (CCS)}}.
  \bibinfo{publisher}{ACM}, \bibinfo{pages}{135--147}.
\newblock


\bibitem[\protect\citeauthoryear{Papernot, McDaniel, Goodfellow, Jha, Celik,
  and Swami}{Papernot et~al\mbox{.}}{2017}]%
        {papernot2017practical}
\bibfield{author}{\bibinfo{person}{Nicolas Papernot}, \bibinfo{person}{Patrick
  McDaniel}, \bibinfo{person}{Ian Goodfellow}, \bibinfo{person}{Somesh Jha},
  \bibinfo{person}{Z~Berkay Celik}, {and} \bibinfo{person}{Ananthram Swami}.}
  \bibinfo{year}{2017}\natexlab{}.
\newblock \showarticletitle{Practical black-box attacks against machine
  learning}. In \bibinfo{booktitle}{\emph{Proceedings of the ACM on Asia
  conference on computer and communications security (Asia CCS)}}.
  \bibinfo{publisher}{ACM}, \bibinfo{pages}{506--519}.
\newblock


\bibitem[\protect\citeauthoryear{Rakin, He, and Fan}{Rakin
  et~al\mbox{.}}{2019}]%
        {rakin2019bit}
\bibfield{author}{\bibinfo{person}{Adnan~Siraj Rakin}, \bibinfo{person}{Zhezhi
  He}, {and} \bibinfo{person}{Deliang Fan}.} \bibinfo{year}{2019}\natexlab{}.
\newblock \showarticletitle{Bit-Flip Attack: Crushing Neural Network with
  Progressive Bit Search}. In \bibinfo{booktitle}{\emph{IEEE/CVF International
  Conference on Computer Vision (ICCV)}}. \bibinfo{publisher}{IEEE},
  \bibinfo{pages}{1211--1220}.
\newblock


\bibitem[\protect\citeauthoryear{Reagen, Gupta, Pentecost, Whatmough, Lee,
  Mulholland, Brooks, and Wei}{Reagen et~al\mbox{.}}{2018}]%
        {reagen2018ares}
\bibfield{author}{\bibinfo{person}{Brandon Reagen}, \bibinfo{person}{Udit
  Gupta}, \bibinfo{person}{Lillian Pentecost}, \bibinfo{person}{Paul
  Whatmough}, \bibinfo{person}{Sae~Kyu Lee}, \bibinfo{person}{Niamh
  Mulholland}, \bibinfo{person}{David Brooks}, {and} \bibinfo{person}{Gu-Yeon
  Wei}.} \bibinfo{year}{2018}\natexlab{}.
\newblock \showarticletitle{Ares: A framework for quantifying the resilience of
  deep neural networks}. In \bibinfo{booktitle}{\emph{55th ACM/ESDA/IEEE Design
  Automation Conference (DAC)}}. \bibinfo{publisher}{IEEE},
  \bibinfo{pages}{1--6}.
\newblock


\bibitem[\protect\citeauthoryear{Reagen, Whatmough, and Adolf}{Reagen
  et~al\mbox{.}}{2016}]%
        {reagen2016minerva}
\bibfield{author}{\bibinfo{person}{Brandon Reagen}, \bibinfo{person}{Paul
  Whatmough}, {and} \bibinfo{person}{et.~al. Adolf}.}
  \bibinfo{year}{2016}\natexlab{}.
\newblock \showarticletitle{Minerva: Enabling low-power, highly-accurate deep
  neural network accelerators}. In \bibinfo{booktitle}{\emph{ACM/IEEE 43rd
  Annual International Symposium on Computer Architecture (ISCA)}}.
  \bibinfo{publisher}{IEEE}, \bibinfo{pages}{267--278}.
\newblock


\bibitem[\protect\citeauthoryear{Romailler and Pelissier}{Romailler and
  Pelissier}{2017}]%
        {romailler2017practical}
\bibfield{author}{\bibinfo{person}{Yolan Romailler} {and}
  \bibinfo{person}{Sylvain Pelissier}.} \bibinfo{year}{2017}\natexlab{}.
\newblock \showarticletitle{Practical fault attack against the Ed25519 and
  EdDSA signature schemes}. In \bibinfo{booktitle}{\emph{Workshop on Fault
  Diagnosis and Tolerance in Cryptography (FDTC)}}. \bibinfo{publisher}{IEEE},
  \bibinfo{pages}{17--24}.
\newblock


\bibitem[\protect\citeauthoryear{Schorn, Elsken, Vogel, Runge, Guntoro, and
  Ascheid}{Schorn et~al\mbox{.}}{2020}]%
        {schorn2019automated}
\bibfield{author}{\bibinfo{person}{Christoph Schorn}, \bibinfo{person}{Thomas
  Elsken}, \bibinfo{person}{Sebastian Vogel}, \bibinfo{person}{Armin Runge},
  \bibinfo{person}{Andre Guntoro}, {and} \bibinfo{person}{Gerd Ascheid}.}
  \bibinfo{year}{2020}\natexlab{}.
\newblock \showarticletitle{Automated design of error-resilient and
  hardware-efficient deep neural networks}. In \bibinfo{booktitle}{\emph{Neural
  Computing and Applications (Neural. Comput. Appl.)}}.
  \bibinfo{publisher}{Springer}, \bibinfo{pages}{1--19}.
\newblock


\bibitem[\protect\citeauthoryear{Schorn, Guntoro, and Ascheid}{Schorn
  et~al\mbox{.}}{2018}]%
        {schorn2018efficient}
\bibfield{author}{\bibinfo{person}{Christoph Schorn}, \bibinfo{person}{Andre
  Guntoro}, {and} \bibinfo{person}{Gerd Ascheid}.}
  \bibinfo{year}{2018}\natexlab{}.
\newblock \showarticletitle{Efficient On-Line Error Detection and Mitigation
  for Deep Neural Network Accelerators}. In
  \bibinfo{booktitle}{\emph{International Conference on Computer Safety,
  Reliability, and Security (SAFECOMP)}}. \bibinfo{publisher}{Springer},
  \bibinfo{pages}{205--219}.
\newblock


\bibitem[\protect\citeauthoryear{Stallkamp, Schlipsing, Salmen, and
  Igel}{Stallkamp et~al\mbox{.}}{2012}]%
        {Stallkamp2012}
\bibfield{author}{\bibinfo{person}{Johannes Stallkamp}, \bibinfo{person}{Marc
  Schlipsing}, \bibinfo{person}{Jan Salmen}, {and} \bibinfo{person}{Christian
  Igel}.} \bibinfo{year}{2012}\natexlab{}.
\newblock \showarticletitle{Man vs. computer: Benchmarking machine learning
  algorithms for traffic sign recognition}.
\newblock \bibinfo{journal}{\emph{Neural Networks}}  \bibinfo{volume}{32},
  \bibinfo{pages}{323--332}.
\newblock


\bibitem[\protect\citeauthoryear{Tan and Le}{Tan and Le}{2019}]%
        {tan2019efficientnet}
\bibfield{author}{\bibinfo{person}{Mingxing Tan} {and} \bibinfo{person}{Quoc~V.
  Le}.} \bibinfo{year}{2019}\natexlab{}.
\newblock \showarticletitle{EfficientNet: Rethinking Model Scaling for
  Convolutional Neural Networks}. In \bibinfo{booktitle}{\emph{Proceedings of
  the 36th International Conference on Machine Learning (ICML)}},
  Vol.~\bibinfo{volume}{97}. \bibinfo{publisher}{PMLR},
  \bibinfo{pages}{6105--6114}.
\newblock


\bibitem[\protect\citeauthoryear{Xia, Liu, Ning, Chakrabarty, and Wang}{Xia
  et~al\mbox{.}}{2017}]%
        {xia2017fault}
\bibfield{author}{\bibinfo{person}{Lixue Xia}, \bibinfo{person}{Mengyun Liu},
  \bibinfo{person}{Xuefei Ning}, \bibinfo{person}{Krishnendu Chakrabarty},
  {and} \bibinfo{person}{Yu Wang}.} \bibinfo{year}{2017}\natexlab{}.
\newblock \showarticletitle{Fault-tolerant training with on-line fault
  detection for RRAM-based neural computing systems}. In
  \bibinfo{booktitle}{\emph{Proceedings of the 54th Annual Design Automation
  Conference 2017 (DAC)}}. \bibinfo{publisher}{ACM},
  \bibinfo{pages}{33:1--33:6}.
\newblock


\bibitem[\protect\citeauthoryear{Yan, Shi, Liao, Hashimoto, Zhou, and Zhuo}{Yan
  et~al\mbox{.}}{2020}]%
        {yan2020single}
\bibfield{author}{\bibinfo{person}{Zheyu Yan}, \bibinfo{person}{Yiyu Shi},
  \bibinfo{person}{Wang Liao}, \bibinfo{person}{Masanori Hashimoto},
  \bibinfo{person}{Xichuan Zhou}, {and} \bibinfo{person}{Cheng Zhuo}.}
  \bibinfo{year}{2020}\natexlab{}.
\newblock \showarticletitle{When Single Event Upset Meets Deep Neural Networks:
  Observations, Explorations, and Remedies}. In \bibinfo{booktitle}{\emph{25th
  Asia and South Pacific Design Automation Conference (ASP-DAC)}}.
  \bibinfo{publisher}{IEEE}, \bibinfo{pages}{163--168}.
\newblock


\bibitem[\protect\citeauthoryear{Yang and Murmann}{Yang and Murmann}{2017}]%
        {yang2017sram}
\bibfield{author}{\bibinfo{person}{Lita Yang} {and} \bibinfo{person}{Boris
  Murmann}.} \bibinfo{year}{2017}\natexlab{}.
\newblock \showarticletitle{SRAM voltage scaling for energy-efficient
  convolutional neural networks}. In \bibinfo{booktitle}{\emph{18th
  International Symposium on Quality Electronic Design (ISQED)}}.
  \bibinfo{publisher}{IEEE}, \bibinfo{pages}{7--12}.
\newblock


\bibitem[\protect\citeauthoryear{Yao, Rakin, and Fan}{Yao
  et~al\mbox{.}}{2020}]%
        {yao2020deephammer}
\bibfield{author}{\bibinfo{person}{Fan Yao}, \bibinfo{person}{Adnan~Siraj
  Rakin}, {and} \bibinfo{person}{Deliang Fan}.}
  \bibinfo{year}{2020}\natexlab{}.
\newblock \showarticletitle{DeepHammer: Depleting the Intelligence of Deep
  Neural Networks through Targeted Chain of Bit Flips}. In
  \bibinfo{booktitle}{\emph{29th USENIX Security Symposium (USENIX Security)}}.
  \bibinfo{publisher}{USENIX Association}, \bibinfo{pages}{1463--1480}.
\newblock


\bibitem[\protect\citeauthoryear{Zhao, Wang, Gongye, Wang, Fei, and Lin}{Zhao
  et~al\mbox{.}}{2019}]%
        {zhao2019fault}
\bibfield{author}{\bibinfo{person}{Pu Zhao}, \bibinfo{person}{Siyue Wang},
  \bibinfo{person}{Cheng Gongye}, \bibinfo{person}{Yanzhi Wang},
  \bibinfo{person}{Yunsi Fei}, {and} \bibinfo{person}{Xue Lin}.}
  \bibinfo{year}{2019}\natexlab{}.
\newblock \showarticletitle{Fault Sneaking Attack: a Stealthy Framework for
  Misleading Deep Neural Networks}. In \bibinfo{booktitle}{\emph{Proceedings of
  the 56th Annual Design Automation Conference 2019 (DAC)}}.
  \bibinfo{publisher}{ACM}, \bibinfo{pages}{165}.
\newblock


\bibitem[\protect\citeauthoryear{Zhezhi, Adnan, Jingtao, Chaitali, and
  Fan}{Zhezhi et~al\mbox{.}}{2020}]%
        {He2020cvpr}
\bibfield{author}{\bibinfo{person}{He Zhezhi}, \bibinfo{person}{Rakin Adnan,
  Siraj}, \bibinfo{person}{Li Jingtao}, \bibinfo{person}{Chakrabarti Chaitali},
  {and} \bibinfo{person}{Deliang Fan}.} \bibinfo{year}{2020}\natexlab{}.
\newblock \showarticletitle{Defending and Harnessing the Bit-Flip based
  Adversarial Weight Attack}. In \bibinfo{booktitle}{\emph{IEEE/CVF Conference
  on Computer Vision and Pattern Recognition (CVPR)}}.
  \bibinfo{publisher}{IEEE}, \bibinfo{pages}{14083--14091}.
\newblock


\end{thebibliography}

\end{document}